\documentclass[doublecolumn]{IEEEtran}

\usepackage{graphicx}
\usepackage{subfig}
\usepackage{amsmath}
\usepackage{amsthm} 
\usepackage{amssymb}
\usepackage[font=footnotesize]{caption} 
\usepackage{subfig} 
\usepackage[noadjust]{cite} 
\usepackage{color}
\usepackage{algorithm} 
\usepackage{algpseudocode} 
\usepackage{enumerate}
\usepackage[hidelinks,colorlinks=false]{hyperref}
\usepackage[left=0.75in, right=0.75in, top=0.75in, bottom=0.75in]{geometry}
\usepackage{authblk} 
\usepackage{arydshln} 
\usepackage{booktabs}
\usepackage{multirow}

\usepackage{bm}
\usepackage{tikz}
\usetikzlibrary{calc} 
\usetikzlibrary{shapes} 
\usetikzlibrary{chains}
\usetikzlibrary{fit}
\usetikzlibrary{arrows}
\usetikzlibrary{decorations.text} 
\usetikzlibrary{decorations.markings}
\usetikzlibrary{decorations.pathmorphing} 
\usetikzlibrary{shadows}
\usetikzlibrary{patterns}
\usetikzlibrary{matrix}
\usepackage{pgfplots}
\usepackage[europeanresistors]{circuitikz}
\usepackage[outline]{contour} 
\contourlength{1.5pt}

\newtheorem{theorem}{Theorem}

\graphicspath{{figures/}}

\newcommand{\blue}{\textcolor{black}}

\begin{document}

\title{PreGME: Prescribed Performance Control of Aerial Manipulators based on Variable‑Gain ESO}
\author{Mengyu Ji, Shiliang Guo, Zhengzhen Li, Jiahao Shen, Huazi Cao, Shiyu Zhao
 \thanks{
 Mengyu Ji and Zhengzhen Li are with the College of Computer Science and Technology, Zhejiang University, Hangzhou 310058, China, and with the WINDY Lab, Department of Artificial Intelligence, Westlake University, Hangzhou 310030, China (E-mail:\{jimengyu, lizhengzhen\}@westlake.edu.cn).}
\thanks{Shiliang Guo, Jiahao Shen are with the WINDY Lab, Department of Artificial Intelligence, Westlake University, Hangzhou 310030, China (E-mail: \{guoshiliang, shenjiahao\}@westlake.edu.cn).}
\thanks{Huazi Cao is with the College of Artificial Intelligence, Zhejiang University, Hangzhou 310058, China (E-mail:caohuazi@wioe.westlake.edu.cn).}
\thanks{Shiyu Zhao is with the WINDY Lab, Department of Artificial Intelligence, and with the Research Center for Industries of the Future, Westlake University, Hangzhou 310030, China (E-mail: zhaoshiyu@westlake.edu.cn).}
\thanks{(Corresponding author: Huazi Cao)}
}

\maketitle
\begin{abstract}
An aerial manipulator, comprising a multirotor base and a robotic arm, is subject to significant dynamic coupling between these two components. Therefore, achieving precise and robust motion control is a challenging yet important objective.
Here, we propose a novel prescribed performance motion control framework based on variable‑gain extended state observers (ESOs), named PreGME. The method includes variable‑gain ESOs for real-time estimation of dynamic coupling and a prescribed performance flight control that incorporates error trajectory constraints. Compared with existing methods, the proposed approach exhibits the following two characteristics. First, the adopted variable‑gain ESOs can accurately estimate rapidly varying dynamic coupling. This enables the proposed method to handle manipulation tasks that require aggressive motion of the robotic arm. Second, by prescribing the performance, a preset error trajectory is generated to guide the system evolution along this trajectory. This strategy allows the proposed method to ensure the tracking error remains within the prescribed performance envelope, thereby achieving high-precision control. Experiments on a real platform, including aerial staff twirling, aerial mixology, and aerial cart-pulling experiments, are conducted to validate the effectiveness of the proposed method. Experimental results demonstrate that even under the dynamic coupling caused by rapid robotic arm motion  (end-effector velocity: 1.02 m/s, acceleration: 5.10 m/s$^2$), the proposed method achieves high tracking performance.

\end{abstract}


\section{Introduction}
An aerial manipulator consists of a multirotor base and a robotic arm, thereby combining aerial mobility with dexterous manipulation capabilities \cite{cao2025proximal}. In contrast to ground-based robotic arms, aerial manipulators extend the workspace into a more expansive three-dimensional space. This breakthrough overcomes operational altitude limitations, enabling the execution of complex manipulation tasks at high altitudes. Compared to \blue{standard} multirotors, aerial manipulators can significantly improve aerial operation accuracy by leveraging their robotic arm to compensate for the multirotor base's tracking errors. Due to these advantages, aerial manipulators exhibit high potential in fields such as contact-based inspection \cite{bodie2020active,nava2019direct}, aerial grasping \cite{2025_aerial_grasping_yadav},  and aerial assembly \cite{2022_shanghaijioatong_assembly}.

\begin{figure}[!t]
\centering
\includegraphics[width=\linewidth]{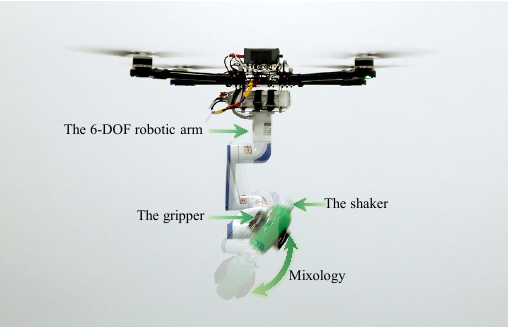}
\caption{Aerial mixology by an aerial manipulator. 
The experimental video is available at https://youtu.be/LGKp5rYE8GU.}
\label{fig_main_mixology}
\end{figure}

Current research on aerial manipulators has encompassed several critical directions, such as motion control \cite{bodie2021dynamic}, motion planning \cite{cao2024motion}, and manipulation control \cite{2024_extracting_object}. These studies enhance the capability of aerial manipulators to perform complex manipulation tasks through advancements in control and motion planning. The motion control is important since it is not only a prerequisite for ensuring stable operation but also the critical link that connects high-level planning to physical execution. In addition, the performance of the motion control, encompassing accuracy, responsiveness, and robustness, directly determines the upper limit of the entire system's task execution capability.

Over the past decade, a wide range of motion control approaches have been developed, which can generally be classified into coupled control and decoupled control methods. In the coupled control method, the multirotor base and the robotic arm are treated as a unified entity \cite{2014_couple_control}. 
Such an approach requires a completely accurate dynamic model that includes the multirotor base, the robotic arm, and the dynamic coupling between them. However, accurate models are often difficult to obtain in practice. Additionally, the lightweight robotic arm's servo motors may not provide sufficient feedback signals for constructing an accurate model. The decoupled control method, in contrast, treats the multirotor and the robotic arm as separate subsystems and regards their interaction as disturbances. It enables the control design of the two subsystems to be conducted separately. This allows the decoupled control method to operate without requiring a precise and complete dynamic model of the aerial manipulator. 
This advantage allows the decoupled control method to impose lower hardware requirements and offer a broader range of applications \cite{bodie2021dynamic}. Existing research on the decoupled control method has primarily focused on two aspects: dynamic coupling estimation and flight control design.

The dynamic coupling in aerial manipulators arises from the mutual influence of motion between the multirotor base and the robotic arm. In early studies on decoupled control, the dynamic coupling between the multirotor base and the robotic arm was often neglected, with independent controllers designed separately for the multirotor base and the robotic arm \cite{zhang2018grasp}. This approach is applicable when the robotic arm is relatively light compared to the multirotor base. Nevertheless, as the robotic arm's mass increases, the dynamic coupling becomes more significant, which may cause performance deterioration or unstable behavior. To compensate for the influence of dynamic coupling, mathematical models of the coupling have been introduced into the control design to improve the control accuracy of aerial manipulators \cite{bodie2021dynamic}. However, acquiring a precise mathematical model in practical applications is challenging due to the inherent complexity of real-world systems. To address this issue, disturbance observers that do not rely on precise models have been incorporated into the control of aerial manipulators, such as the extended state observer \cite{cao2023eso} and the sliding mode disturbance observer \cite{chen2022adaptive}. Nevertheless, these disturbance observers remain prone to noise amplification and inadequate response to nonlinear disturbances due to their linear high-gain structures \cite{2017_TIE_NESO}. Therefore, in the presence of high-frequency disturbances or rapid robotic arm motion, the use of such disturbance observers may lead to notable estimation errors.

Flight control design is also a critical issue in the research of aerial manipulators. Initially, the traditional PID control method is directly employed to design the flight control for the multirotor base \cite{mellinger2011design_grasping}. However, this method did not effectively utilize known model information to counteract dynamic coupling. To further improve the control accuracy of the multirotor base, nonlinear control methods such as sliding mode control \cite{2015_slide_mode_control}, adaptive control \cite{liu2021ddpg}, and backstepping control \cite{jimenez2015aerial} have been applied to the flight control. Although these methods can guarantee stability, \blue{they cannot ensure the tracking error within the prescribed performance envelope when subjected to dynamic coupling disturbances \cite{CAO_practical_ppt}.}  To achieve this objective, prescribed performance control (PPC) has been adopted \cite{kong2023experimental_ppc}. \blue{According to \cite{preset-trajectory}, the existing PPC methods can be roughly divided into three types: funnel control, nonlinear mapping-based PPC, and barrier Lyapunov function-based PPC. However, all these methods suffer from potential singularity issues due to their reliance on barrier functions \cite{preset-trajectory}.} Specifically, subject to intense external disturbances, the tracking error may approach the boundary of the barrier functions. \blue{In such cases, the values of barrier functions grow to infinity. This consequently leads to an unbounded control input reference, causing actuator saturation and potential system instability  \cite{preset-trajectory}}. To mitigate this, smooth non-singular mappings \cite{lei2023_steer_mapping_ppc} and dynamic boundary adjustments \cite{yong2020flexible} have been proposed. Both strategies, however, typically involve the calculation of high-order derivatives of the transformed error, which may complicate controller design and increase tuning difficulty.

To overcome the aforementioned challenges in dynamic coupling estimation and flight control design, \blue{this paper proposes a motion control strategy based on variable-gain extended state observers (ESOs) and error trajectory constraints, which we refer to as PreGME. This strategy is inspired by a framework called GME, established in \cite{cao2023eso}, which synergizes geometric control with explicit modeling for measurable dynamic coupling disturbances and estimation for unmeasurable disturbances. By integrating error trajectory constraints into this framework, PreGME can achieve high-precision prescribed performance control under strong dynamic coupling conditions. Three experiments on practical platforms are conducted to demonstrate the effectiveness of the proposed algorithm.} The main contributions of this paper include:

1) We propose a partially decoupled motion control framework based on variable-gain ESOs \cite{2021CESO} for high-precision motion control of aerial manipulators. The method utilizes ESOs to estimate dynamic couplings, enabling accurate estimation of rapidly varying dynamic couplings using only limited measurements. Compared to motion control approaches employing the extended state observer \cite{cao2023eso} or the sliding mode observer \cite{chen2022adaptive}, the proposed framework is more suitable for application scenarios involving high-speed robotic arm motion.

2) We propose a novel prescribed performance flight control method based on error trajectory constraints. The constraints are defined by a user-specified preset error trajectory based on the prescribed performance envelope. A nonlinear control law is then designed to force the actual tracking error to follow this error trajectory, thereby ensuring that the tracking error of the aerial manipulator satisfies prescribed performance requirements. Unlike existing PPC methods (e.g., \cite{chen2022adaptive,liang2022adaptive}), \blue{the proposed approach employs a smooth and non-singular preset error trajectory rather than a barrier function to constrain the tracking error, thereby avoiding the occurrence of singularity issues.} Moreover, the preset error trajectory can be flexibly designed, enabling the controller to adapt to task scenarios with varying performance specifications.

\section{\blue{Preliminaries and Control System Overview}}

This section introduces \blue{some necessary preliminaries used in this paper and provides an overview of the control system}.

\subsection{Dynamics of the Aerial Manipulator }

In this work, the aerial manipulator consists of a quadcopter base and a six-degree-of-freedom serial robotic arm (see Fig.~\ref{fig_control_frame}). The robotic arm is mounted underneath the quadcopter base. The robotic arm's end-effector can achieve both translational and rotational motion relative to the quadcopter base under the control of six joints. As shown in Fig.~\ref{fig_control_frame}, two coordinate frames are employed: the north-east-down inertial frame $\varSigma_I$ and the body-fixed frame of the quadcopter $\varSigma_B$.

Let $\bm p \in \mathbb{R}^{3}$ and $\bm v \in \mathbb{R}^{3}$ denote the position and velocity of the quadcopter base in $\varSigma_I$, respectively. The angular velocity of the quadcopter base in $\varSigma_B$ is denoted as $\bm \omega \in \mathbb{R}^{3}$, and $\bm R \in SO(3)$ represents the rotation matrix from $\varSigma_B$ to $\varSigma_I$. Let $ m_B \in \mathbb{R} $ and $ m_R \in \mathbb{R}$ be the mass of the quadcopter and the robotic arm, respectively. The inertia matrix of the quadcopter base in $\varSigma_B$ is denoted by $\bm{I} \in \mathbb{R}^{3 \times 3}$. The total thrust of the rotors is denoted by $T \in \mathbb{R}$. The torque vector of the rotors in $\varSigma_B $ is represented by $\bm{\tau} \in \mathbb{R}^{3}$. The dynamic coupling force and torque exerted by the robotic arm on the quadcopter base are denoted by $\bm{\Delta}_v \in \mathbb{R}^{3}$ and $\bm{\Delta}_\omega \in \mathbb{R}^3$, respectively. According to \cite{zhang2019robust,cao2023eso}, the dynamics of the quadcopter base is
\begin{equation}
\label{eq_dynamics}
\begin{aligned}
\text{Position dynamics}\left\{\begin{aligned}
&\dot{\bm{p}}=\bm{v},\\
&\dot{\bm v}=-\frac{T \bm R \bm n}{m_B+m_R}+g\bm n+\bm{\Delta}_v,
\end{aligned}\right.\\
\text{Attitude dynamics}\left\{\begin{aligned}
&\dot{\bm R}=\bm R[\bm \omega]_{\times},\\
&\dot{\bm \omega}=\bm I^{-1}(\bm \tau-\bm \omega\times \bm I\bm \omega)+\bm \Delta_\omega,
\end{aligned}\right.\\
\end{aligned}
\end{equation}
where $\bm n=[0,0,1]^T$ is a unit vector, $g \in \mathbb{R}$ is the gravitational acceleration, and $[\cdot]_{\times}$ denotes the skew-symmetric transformation. \blue{The control force vector is defined as $\bm T = T \bm R \bm n \in \mathbb{R}^3$.} The expressions of $\bm{\Delta}_v$ and $\bm{\Delta}_\omega$ can be found in \cite{zhang2019robust,cao2023eso}. Since their expressions involve quantities that cannot be accurately measured by sensors, such as the angular acceleration of the quadcopter base and the joint angular accelerations of the robotic arm, these dynamic coupling terms are treated as unknown disturbances in this paper. When the robotic arm is absent, i.e., $m_R=0$, $\bm{\Delta}_v =0$, and $\bm{\Delta}_\omega =0$, model~\eqref{eq_dynamics} degenerates into a standard quadrotor dynamics.

\subsection{Variable-Gain Extended State Observer}\label{Sec_Variable-Gain_ESO}

This subsection introduces the \blue{variable-gain} ESO~\cite{2021CESO} for estimating dynamic couplings. 
\blue{We adopt the variable-gain ESO because it provides faster dynamic response and stronger rejection of high-frequency noise. This advantage originates from its design, which differs from traditional fixed-gain ESOs by using a nonlinear variable-gain scheme\cite{2021CESO}.} The stability analysis and other details of the variable-gain ESO can be found in~\cite{2021CESO}. 

Consider a system with the input $u \in \mathbb{R}$ and the output $x \in \mathbb{R}$ in the form $\dot x = \Delta + u$, where $\Delta \in \mathbb{R}$ is the unknown term of the system. The output $x$ can be measured by sensors. \blue{ Denote $y_1 = x$ and  $y_2 = \Delta \in \mathbb{R}$} as the state and the extended state of the ESO, respectively. Then the system can be described as 
\blue{\begin{equation}
\label{eq_CESO_system}
\dot y_1 = y_2 + u, \ \dot y_2 = \dot \Delta.
\end{equation}}
Define $h \in \mathbb{R}$ as an internal auxiliary state of the variable-gain ESO. According to \cite{2021CESO}, the variable-gain ESO is designed as  
\begin{equation}
\label{eq_CESO}
\begin{aligned}
&\dot h= \frac{\alpha g(e)}{\varepsilon} + u, \quad e = y_1 - h, \quad \hat y_2 = \frac{\alpha g(e)}{\varepsilon},
\end{aligned}
\end{equation}
where $\hat{y}_2 \in \mathbb{R}$ is the estimates of $y_2$, $e \in \mathbb{R}$ is the error between $y_1$ and $h$, $\alpha>0$ is a constant, and $\varepsilon \in (0,1)$ is a tuning parameter, $g(e) \in \mathbb{R}$ is a variable gain and is designed as  
\begin{equation} 
\label{eq_function_g}
g (e)=  \frac{\exp{(e)}+\exp{(-e)}}{w(\exp{(e)}+\exp{(-e)})+d} \, e,
\end{equation}
where $w>0$ and $d>0$ are constants. \blue{The rationale for this gain function is twofold. First, it inherently balances noise suppression and state reconstruction by assigning low gain to small errors for noise mitigation and high gain to large errors for rapid convergence. Second, its parameters offer flexible tuning: increasing $d$ lowers the gain floor to enhance noise immunity, while decreasing $w$ raises the gain ceiling to accelerate dynamic response. }

\blue{Based on the separation principle \cite{atassi2001separation} and Theorem~1 in \cite{2021CESO}, we conclude that,} assuming the unknown term $\Delta$ and its derivative $\dot \Delta$ are bounded, if the nonlinear gain function is selected as \eqref{eq_function_g},  then for any given $\delta_f > 0$ and $t_f > 0$, there exists a constant $\varepsilon^* \in (0,1) $ such that $\forall \varepsilon \in (0, \varepsilon^*)$, $|\Delta - \hat \Delta| \leq \delta_f, \forall t \ge t_f$.

\subsection{Tracking Error Trajectory Constraints}\label{sec_preset_error_trajectory}

\begin{figure*}[!t]
\centering
\includegraphics[width=\linewidth]{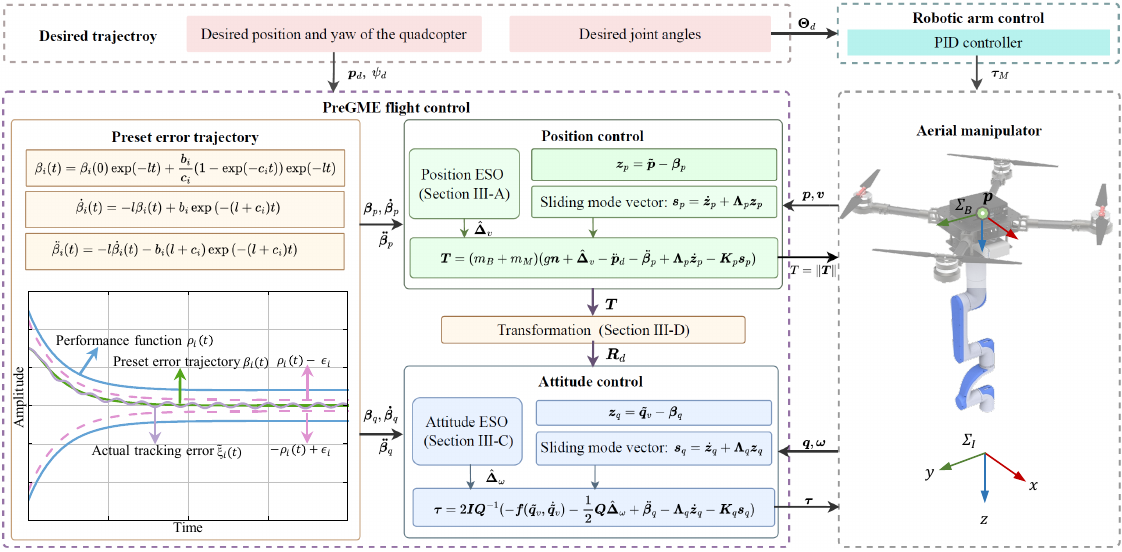}
\caption{The proposed control framework of the aerial manipulator system.}
\label{fig_control_frame}
\end{figure*}

Inspired by \cite{preset-trajectory, preset-trajectory2025}, the tracking error trajectory constraints are characterized by a preset error trajectory. \blue{The preset error trajectory is chosen for its ability to guide the controller's tracking behavior. This ensures that the actual tracking error can remain within the prescribed performance envelope (see Fig.~\ref{fig_control_frame}). }
The construction of the preset error trajectory involves two steps.

The first step is to design a performance envelope to describe the prescribed tracking performance. 
Let $\bm \rho(t) \in \mathbb{R}^3$  denote the upper bound function of the performance envelope. It is designed as  $\bm \rho(t) = (\bm \rho_{0} - \bm \rho_{\infty}) \exp(-l t) + \bm \rho_{\infty}$,
where $\bm \rho_{0} = [\rho_{0,1}, \rho_{0,2}, \rho_{0,3}]^T \in \mathbb{R}^3$, $\bm \rho_{\infty} = [\rho_{\infty,1}, \rho_{\infty,2}, \rho_{\infty,3}]^T \in \mathbb{R}^3$, and $l >0$ are parameters to be determined. Then, the lower bound function of the performance envelope is $-\bm \rho(t)$. Let $\bm \xi \in \mathbb{R}^3$ and \blue{$\tilde{\bm \xi} \in \mathbb{R}^3$ denote the state and the state error}, respectively. The parameter $\bm \rho_0$ can be designed based on the initial tracking error $\tilde{\bm \xi}(0)$ to ensure that the initial error is within the performance envelope, i.e., $\rho_{0,i} > |\tilde{\xi}_i(0)|, i=1,2,3$. The exponential decay term $\exp(-l t)$ controls the convergence speed, where a larger $l$ yields faster convergence. The value of $l$ is determined by the task completion time. The parameter $\bm \rho_{\infty}$ represents the maximum allowable steady-state error. Its lower bound is typically determined by sensor noise or modeling uncertainties, while the upper bound is determined by the task-specific accuracy requirements. 

The second step is to obtain the preset error trajectory based on the performance envelope. Define $\bm \beta(t) = [\beta_1(t), \beta_2(t), \beta_3(t)]^T \in \mathbb{R}^3$ as the preset error trajectory. Define $\bm c =[c_1,c_2,c_3]^T\in \mathbb{R}^3$ and $\bm b = l \tilde{\bm \xi}(0) + \dot{\tilde{\bm \xi}}(0)\in \mathbb{R}^3$ as the parameter vectors of $\bm \beta(t)$. According to \cite{preset-trajectory}, the definition of $\bm \beta(t)$  is   
\begin{equation}
\label{eq_beta}
\beta_{i}(t) = \beta_{i}(0)\exp(-l t) + \frac{b_i}{c_i} (1 - \exp(-c_i t)) \exp(-l t), 
\end{equation}
where $\beta_{i}(0) = \tilde{\xi}_i(0)$, and $i=1,2,3$. \blue{ The preset error trajectory in \eqref{eq_beta} is smooth and singularity-free. Moreover, appropriate selection of $ c_i,i=1,2,3$ ensures that $ \beta_{i}(t)$ remains within the performance envelope.} 

Let ${\epsilon}_i \in \mathbb{R}$ be a positive constant  satisfying  that  $\epsilon_i < \min\{\rho_{\infty,i}, \rho_{0,i} - |\tilde{\xi}_i(0)|\}$ for $i = 1,2,3$. \blue{According to Lemma 2 in \cite{preset-trajectory}}, for the preset tracking error $\beta_{i}(t)$ in \eqref{eq_beta}, there always exists a design parameter $c_i > |b_i|/(\rho_{0,i} - |\tilde{\xi}_i(0)| - \epsilon_i)$ such that $|\beta_{i}(t)| < \rho_i(t) - \epsilon_i, \ \forall t \ge 0$.

\subsection{Control System Overview}

The overall architecture of the motion control system is shown in Fig.~\ref{fig_control_frame}. 
The system is decomposed into two components.

The first is the flight control of the quadcopter base. 
The inputs of this part are the desired position $\bm p_d$ and the desired yaw angle $\psi_d$, and the outputs are the thrust $T$ and the torque $\bm \tau$. 
The flight control adopts a cascaded architecture comprising position control (see Section~\ref{sec_pos_control}) and attitude control (see Section~\ref{sec_att_control}). In both position and attitude control, a variable-gain ESO is employed to estimate the dynamic coupling generated by the robotic arm's motion in real time. The estimated values of the dynamic coupling are then fed back in real time to their respective position and attitude control laws.

The second is the joint control of the robotic arm. The input of this part is the desired joint angle $\bm \Theta_d\in\mathbb{R}^6$, and the output is the joint torque $\bm \tau_M\in\mathbb{R}^6$. \blue{The proposed framework is compatible with a variety of control methods. In this work, we adopt a PID-based approach for joint control. Other methods capable of regulating joint torque to track the desired joint angle are also applicable.}

\section{Prescribed Performance Control based on Variable-Gain ESO}\label{sec_tracking_control}

\blue{This section proposes a variable-gain ESO-based prescribed performance for flight control of the aerial manipulator. The proposed flight control consists of position and attitude control. The variable-gain ESOs are employed to estimate the dynamic coupling force and torque. Furthermore, the preset error trajectories are used to constrain the tracking performance, ensuring that the tracking error remains within the performance envelope.}

\subsection{Position ESO} \label{sec_pos_ESO}
\blue{The objective of the position ESO is to estimate the dynamic coupling force $\bm \Delta_v$ based on the variable-gain ESO method described in Section~\ref{Sec_Variable-Gain_ESO}}.  \blue{From \eqref{eq_dynamics}, the position dynamics can be rewritten in a compact form as 
\begin{equation}
\label{eq_pos_dynamics}
\begin{aligned}
\dot {\bm p} = \bm v,\quad \dot {\bm v}=\bm u_v +\bm{\Delta_{v}},
\end{aligned}
\end{equation}
where $\bm u_v = g\bm n - \bm T/(m_B +m_R) \in \mathbb{R}^3$. $\bm \Delta_v$ is treated as an unknown term and estimated by the position ESO. Let  $\hat{\bm \Delta}_v$ be the estimate of $\bm \Delta_v$.
Since the velocity $\bm v$ is available from onboard sensors, it is utilized to estimate $\bm \Delta_v$. Let $\bm v$ and $\bm \Delta_v$ be the state and the extended state of the variable-gain ESO, respectively.  
Consider the system $\dot {\bm v}=\bm u_v +\bm{\Delta_{v}}$}. For the $i$-th element $\hat {\Delta}_{v,i}$ of $\hat{\bm \Delta}_v$, the corresponding ESO is designed as
\begin{equation}
\begin{split}
    \dot h_{v,i}={\alpha_{v,i}g(e_{v,i})}/{\varepsilon_{v,i}}&+u_{v,i}, \quad e_{v,i}=v_i-h_{v,i},\\
\hat \Delta_{v,i}=&{\alpha_{v,i}g(e_{v,i})}/{\varepsilon_{v,i}},
\end{split}
\end{equation}
where  $\bm h_{v} =[h_{v,1},h_{v,2},h_{v,3}]^T \in \mathbb{R}^3$ is an internal auxiliary state,  the gain function $g(e_{v,i})$ is given in \eqref{eq_function_g}, \blue{$\bm e_v = [e_{v,1}, e_{v,2},e_{v,3}]^T \in \mathbb{R}^3$ is the error vector between the velocity $\bm v$ and the auxiliary state $\bm h_v$,} $\bm \alpha_v= [\alpha_{v,1},\alpha_{v,2},\alpha_{v,3}]^T\in \mathbb{R}^3$ is a positive constant vector, and $\bm \varepsilon_v = [\varepsilon_{v,1},\varepsilon_{v,2},\varepsilon_{v,3}]^T\in \mathbb{R}^3$ is a tuning parameter vector. The initial condition of the variable-gain ESO is set as $\bm h_v(0)=\bm v(0)$. 

\blue{Let $\tilde{\bm \Delta}_v = \bm \Delta_v - \hat{\bm \Delta}_v$ denote the estimation error of the position ESO. According to Section~\ref{Sec_Variable-Gain_ESO}, the estimation error satisfies $\| \tilde{\bm \Delta}_v \|  = \| \bm \Delta_{v}-\hat{\bm \Delta}_{v}\|\le \delta_{v,f}$ for all $t\ge t_{v,f}$, where $\delta_{v,f} \in \mathbb{R}$ is a positive constant.}

\subsection{Position Control}\label{sec_pos_control}

The objective of the position control is to compute the control force vector to track the desired position and constrain the tracking error within the performance envelope. The design of the position control law is divided into two steps.

The first is to design a sliding mode vector to guide the error trajectory toward the preset error trajectory.
Denote the quadcopter base's desired position and the position error as $\bm p_d \in \mathbb{R}^3$ and $\tilde{\bm p} = \bm p - \bm p_d \in \mathbb{R}^3$, respectively. The upper bound function for the position performance envelope is denoted by $\bm \rho_p(t) \in \mathbb{R}^3$, and the preset error trajectory is denoted by $\bm \beta_p(t) \in \mathbb{R}^3$, which is designed according to Section~\ref{sec_preset_error_trajectory}. Let $\bm \rho_{p,0}\in \mathbb{R}^3$, $\bm \rho_{p,\infty}\in \mathbb{R}^3$, and $l_p\in \mathbb{R}$ denote the parameters of $\bm \rho_{p}(t)$. Let  $\bm c_p \in \mathbb{R}^3$ and $\bm b_p \in \mathbb{R}^3$ denote the parameters of $\bm \beta_p(t)$.  
Define $\bm z_p(t) =\tilde{\bm p}(t) - \bm \beta_{p}(t) \in \mathbb{R}^3$. Then the sliding mode vector is defined as 
${\bm s}_p =\dot {\bm{z}}_p + \bm{\Lambda}_p {\bm z_p}$, 
where $\bm{\Lambda}_p \in \mathbb{R}^{3\times3}$ is a positive diagonal matrix. 

The second is to design the tracking control law. Taking the time derivative of $\bm s_p$ and substituting \eqref{eq_pos_dynamics} into the resulting expression yields
\begin{equation}
\begin{aligned}
\label{eq_dot_sp}
\dot{\bm s}_p &=\ddot {\bm{z}}_p + \bm{\Lambda}_p\dot {\bm z}_p\\
&= \ddot{\bm p} - \ddot{\bm p}_d  - \ddot{\bm \beta}_{p} + \bm{\Lambda}_p\dot {\bm z}_p\\
& =g\bm n - \frac{\bm T}{m_B + m_R} + \bm \Delta_v- \ddot{\bm p}_d  - \ddot{\bm \beta}_{p} + \bm{\Lambda}_p\dot {\bm z}_p.
\end{aligned}
\end{equation}
Then, the control force input is designed as
\begin{equation}
\label{eq_postion_control}
\bm T = (m_B+m_R) \big( g\bm n+\hat{\bm \Delta}_v-\ddot {\bm p}_{d} -\ddot{\bm{\beta}}_{p} + \bm \Lambda_p \dot {\bm z}_p +\bm K_p \bm s_p\big),
\end{equation}
where $\bm{K}_p \in \mathbb{R}^{3\times3}$ is a positive diagonal matrix.
Then, the total thrust is computed as $T = \|\bm T\|$. 
The proof that the position tracking error $\tilde{\bm p}$ always remains within the performance envelope is given as follows.
\begin{theorem}
\label{thm_pos_error}
Let $\lambda_{\min}(\bm K_p)$ be the minimum eigenvalue of $\bm K_p$, and $[\Lambda_p]_{i,i}$ the $i$-th diagonal element of $\bm \Lambda_p$. Assuming that $\bm \Delta_v$ and its derivative $\dot{\bm \Delta}_v$ are bounded, for the $i$-th element $\tilde{ p}_i(t)$ of the tracking error $\tilde{\bm p}(t)$, if the parameters are chosen such that ${\delta_{v,f}}/([\Lambda_p]_{i,i}\lambda_{\min}(\bm K_p)) < \min\{\rho_{p,\infty,i}, \rho_{p,0,i} - |\tilde{p}_{i}(0)| \}$ and $c_{p,i} > {|b_{p,i}|}/(\epsilon_{p,0,i}-{\delta_{v,f}}/([\Lambda_p]_{i,i}\lambda_{\min}(\bm K_p)))$, then the position tracking error satisfies $ | \tilde{ p}_i(t) |< {\rho}_{p,i}(t), \ \forall t \ge 0$.
\end{theorem}

\begin{proof}
\blue{Substituting the position control law \eqref{eq_postion_control} into \eqref{eq_dot_sp}, the closed-loop dynamics of the sliding mode vector $\bm s_p$ is given by $\dot{\bm s}_p =-\bm K \bm s_p + \tilde{\bm \Delta}_v$.
Define a Lyapunov function as $V = 1/2\bm s^T_p \bm s_p$. The time derivative of $V$ is $\dot{V}  = {\bm s}_p^T \dot{\bm s}_p$. Substituting  \eqref{eq_dot_sp} into the expression of $\dot V$ yields 
\begin{equation}
\begin{aligned}
\label{eq_V_dot_analysis}
\dot{V} &= \bm s^T_p (-\bm K_p \bm s_p + \tilde{\bm \Delta}_{v}) \\
&\leq -\lambda_{\min}(\bm K_p) \|\bm s_p\|^2 + \|\bm s_p\| \|\tilde{\bm \Delta}_v\| \\
& \leq -\lambda_{\min}(\bm K_p) \|\bm s_p\|^2 + \delta_{v} \|\bm s_p\|, 
\end{aligned}
\end{equation}
where $\delta_v$ is the upper bound of the estimation error $\tilde{\bm \Delta}_v$ for $t \ge 0$. According to Theorem 4.18 in \cite{khalil2002nonlinear} and $\bm s_p(0) = \bm 0$, the sliding mode vector satisfies $\|\bm s_p(t)\| \le {\delta_v}/{\lambda_{\min}(\bm K_p)}, \ \forall t \ge 0$.
Consequently, the $i$-th element $s_{p,i}(t)$ of $\bm s_p(t)$ satisfy $|s_{p,i}(t)| \leq \delta_v / \lambda_{\min}(\bm K_p)$. From the definition of $\bm s_p$, we have $s_{p,i} = \dot{z}_{p,i} + [\Lambda_p]_{i,i} z_{p,i}$.} According to Lemma~1 in \cite{lemma_z} and the upper bound of $s_{p,i}(t)$, one has
\begin{equation}
\label{eq_pos_deviation_to_preset_trajectory}
|z_{p,i}(t)| \le {\delta_{v}}/({[\Lambda_p]_{i,i}\lambda_{\min}(\bm K_p)}), \  \forall t\ge 0.
\end{equation}

Further, we prove that $ | \tilde{ p}_i(t) |< {\rho}_{p,i}(t), \ \forall t \ge 0$. 
Combining \eqref{eq_pos_deviation_to_preset_trajectory} with the definition of $\bm z_p$, the position tracking error satisfies
\begin{equation}
\begin{aligned}
|\tilde p_i(t)|&=|\tilde{p}_i(t)+ \beta_{p,i}(t)-  \beta_{p,i}(t)|\\
&\le |\tilde {p}_i(t) - \beta_{p,i}(t)|+| \beta_{p,i}(t)|\\
&\le \frac{\delta_{v}}{[\Lambda_p]_{i,i}\lambda_{\min}(\bm K_p)}  + | \beta_{p,i}(t)|
\end{aligned}
\label{eq_pos_bound_tracking_error}
\end{equation}
as $\forall t\ge 0, \ i\in\{1,2,3\}$. \blue{According to Section~\ref{sec_pos_ESO}, $\delta_v$ is defined as a piecewise constant: $\delta_v = \delta_{v,b}$ for $t < t_{v,f}$ and $\delta_v = \delta_{v,f}$ for $t \ge t_{v,f}$. Here, $\delta_{v,b}$ and $\delta_{v,f}$ represent the upper bounds of the estimation error before and after ESO convergence, respectively. Therefore, $\delta_{v,b}$ is much greater than $\delta_{v,f}$. The control parameters are set such that ${\delta_{v,f}}/([\Lambda_p]_{i,i}\lambda_{\min}(\bm K_p)) < \min\{\rho_{p,\infty,i}, \rho_{p,0,i} - |\tilde{p}_{i}(0)| \}$, and the design parameter of $\beta_{p,i}$  is chosen to satisfy $c_{p,i} > {|b_{p,i}|}/(\rho_{p,0,i} - |\tilde{p}_{i}(0)|-{\delta_{v,f}}/([\Lambda_p]_{i,i}\lambda_{\min}(\bm K_p)))$. Then, according to Section~\ref{sec_preset_error_trajectory}, we obtain $|\beta_{p,i}(t)| < \rho_{p,i}(t) -{\delta_{v,f}}/([\Lambda_p]_{i,i}\lambda_{\min}(\bm K_p))$. For the phase $t \ge t_{v,f}$, applying the resulting inequality of $|\beta_{p,i}(t)|$ to \eqref{eq_pos_bound_tracking_error} yields $|\tilde p_i(t)| <  {\delta_{v,f}}/({[\Lambda_p]_{i,i}\lambda_{\min}(\bm K_p)}) + \rho_{p,i}(t) -{\delta_{v,f}}/([\Lambda_p]_{i,i}\lambda_{\min}(\bm K_p)) = \rho_{p,i}(t)$. This means the tacking error remains within the performance envelope for the phase $t \geq t_{v,f}$. For the phase $t < t_{v,f}$, this inclusion can be guaranteed through the design of the performance envelope parameter. Specifically, we set $\rho_{p,0,i} > ( |\tilde{p}_i(0)| + {\delta_{v,b}}/({[\Lambda]_{i,i}\lambda_{\min}(\bm K_p)}) - \rho_{p,\infty,i} ) \exp(l_p t_{v,f}) + \rho_{p,\infty,i}$. Combining the the specified set of parameters $\rho_{p,0,i}$ and the definition of $\rho_{p,i}(t)$, we obtain $ \rho_{p,i}(t_{v,f}) > |\tilde{p}_i(0)| + {\delta_{v,b}}/({[\Lambda]_{i,i}\lambda_{\min}(\bm K_p)})$. From the monotonic decreasing property of $|\beta_{p,i}(t)|$ and $\rho_{p,i}(t)$, one can conclude that $|\beta_{p,i}(t)| \le |\beta_{p,i}(0)|= |\tilde{p}_i(0)|$ and $\rho_{p,i}(t) > \rho_{p,i}(t_{v,f})$, when $t< t_{v,f}.$ Then, we have $ |\beta_{p,i}(t)| \le |\tilde{p}_i(0)|< \rho_{p,i}(t) -{\delta_{v,b}}/({[\Lambda]_{i,i}\lambda_{\min}(\bm K_p)})$.
Substituting this result into \eqref{eq_pos_bound_tracking_error} yields $|\tilde p_i(t)|< \rho_{p,i}(t)$, when $t < t_{v,f}$. Combining the proof results from both phases, it can be concluded that the tracking error remains within the position performance envelope.} 
\end{proof}

\blue{ Theorem~\ref{thm_pos_error} assumes that $\bm {\Delta}_v$ and $\dot{\bm {\Delta}}_v$ are bounded. This assumption is reasonable because $\bm {\Delta}_v$ and $\dot{\bm {\Delta}}_v$ depend on the motion of the quadcopter base and the robotic arm, which are physically constrained.} The analysis of the influence of parameters on the deviation between the position tracking error and the preset error trajectory is given below. The upper bound of this deviation is given in \eqref{eq_pos_deviation_to_preset_trajectory}. 
It shows that the deviation is inversely proportional to $[\Lambda_p]_{i,i}$ and $\lambda_{\min}(\bm K_p)$, and directly proportional to $\delta_{v,f}$. This implies that there are two ways to reduce the deviation. First, the deviation can be reduced by directly increasing the parameters $\bm{\Lambda}_p$ and $\bm{K_p}$. Second, the deviation can also be reduced by reducing the parameter $\varepsilon_v$ to decrease the upper bound $\delta_{v,f}$ of the position ESO estimation error.

 \subsection{Attitude ESO}\label{sec_att_ESO}

\blue{The objective of the attitude ESO is to estimate the dynamic coupling force $\bm \Delta_\omega$ based on the variable-gain ESO method described in Section~\ref{Sec_Variable-Gain_ESO}. }
\blue{From \eqref{eq_dynamics}, the attitude dynamics can be rewritten in a compact form as 
\begin{equation}
\label{eq_att_dynamics}
    \dot {\bm R} = \bm R [\bm \omega]_{\times},\quad
    \dot{\bm \omega} =\bm u_{\omega} + \bm \Delta_\omega,
\end{equation}
where $\bm u_{\omega} =\bm I^{-1}(\bm \tau-\bm \omega\times \bm I\bm \omega)$.
$\bm \Delta_\omega$ is treated as an unknown term and estimated by the attitude ESO. Let $\hat{\bm \Delta}_\omega \in \mathbb{R}^3$ be the estimate of $\bm \Delta_\omega$.
Since the angular velocity $\bm \omega$ is available from onboard sensors, it is utilized to estimate $\bm \Delta_\omega$. Let $\bm \omega$ and $\bm \Delta_\omega$ be the state and the extended state of the variable-gain ESO, respectively.  
Consider the system $\dot {\bm \omega}=\bm u_\omega +\bm{\Delta_{\omega}}$. } For the $i$-th element $\hat {\Delta}_{\omega,i}$ of $\bm \Delta_\omega$, the corresponding ESO is designed as 
\begin{equation}
    \begin{split}
        \dot h_{\omega,i}= {\alpha _{\omega,i} g(e_{\omega,i})}/{\varepsilon_{\omega,i}} &+ u_{\omega,i}, \quad 
e_{\omega,i}=\omega_i-h_{\omega,i}, \\
 \hat \Delta_{\omega,i}&= {\alpha _{\omega,i} g(e_{\omega,i})}/{\varepsilon_{\omega,i}},
    \end{split}
\end{equation}
where $\bm h_{\omega} = [h_{\omega,1}, h_{\omega,2},h_{\omega,3}]^T\in  \mathbb{R}^3$ is an internal auxiliary state, the gain function $g(e_{\omega,i})$ is given in \eqref{eq_function_g}, \blue{$\bm e_\omega = [e_{\omega,1}, e_{\omega,2},e_{\omega,3}]^T \in \mathbb{R}^3$ is the error vector between $\bm \omega$ and $\bm h_\omega$,} $\bm \alpha_\omega= [\alpha_{\omega,1},\alpha_{\omega,2},\alpha_{\omega,3}]^T\in \mathbb{R}^3$ is a positive constant vector, and $\bm \varepsilon_\omega = [\varepsilon_{\omega,1},\varepsilon_{\omega,2},\varepsilon_{\omega,3}]^T\in \mathbb{R}^3$ is a tuning parameter vector. The initial condition is set as $\bm h_{\omega}(0)= \bm \omega(0)$.

\blue{Let $\tilde{\bm \Delta}_\omega = \bm \Delta_\omega - \hat{\bm \Delta}_\omega$ denote the estimation error of the attitude ESO. According to Section~\ref{Sec_Variable-Gain_ESO}, the estimation error of the attitude ESO is bounded, i.e., $\|\tilde{\bm \Delta}_\omega \| = \| \bm \Delta_{\omega}-\hat{\bm \Delta}_{\omega}\|\le \delta_{\omega,f}, \forall t>t_{\omega,f}$, where $ \delta_{\omega,f} \in \mathbb{R}$ is a positive constant.}

\subsection{Attitude Control}\label{sec_att_control}

The objective of attitude control is to calculate the control torque vector to track the desired attitude and constrain the tracking error within the performance envelope.

Let $\bm R_d=[\bm b_1,\bm b_2,\bm b_3]^T\in SO(3)$ be the desired rotation matrix, and $\psi_d \in \mathbb{R}$ the desired yaw angle. Define $\bm a=[\cos(\psi_d),\sin(\psi_d),0]^T$.  
Then, we have
\begin{equation}
\bm b_3=\frac{\bm T}{T}, \ 
\bm b_2=\frac{  \bm b_3 \times \bm a}{\|\bm b_3 \times \bm a\|}, \ 
\bm b_1=\bm b_2\times \bm b_3.
\end{equation}  
The rotation error is $\tilde{\bm R}=\bm R_d^T\bm R\in SO(3)$, and the corresponding quaternion error is $\tilde {\bm q}=[\tilde q_0, \tilde {\bm q} _v ^T]^T \in \mathbb{R}^{4}$ with 
\begin{equation}
\label{eq_q_error}
\tilde q_0=\frac{1}{2}\sqrt{1+\text {tr}(\tilde {\bm R})},\quad \tilde {\bm q} _v= \frac{1}{4\tilde {q} _0}[\tilde {\bm R}-\tilde {\bm R}^T]^{\vee}, 
\end{equation}
where $\text {tr}(\tilde {\bm R})$ denote the trace of $\tilde {\bm R}$, and $[\cdot]^{\vee}$ denotes the vee map, i.e., the inverse of $[\cdot]_{\times}$.
Define the angular velocity error as $\tilde{\bm{\omega}} = \bm{\omega} - \tilde{\bm{R}}\bm{\omega}_d \in \mathbb{R}^3$, where $\bm{\omega}_d = [\bm{R}^T\dot{\bm{R}}_d]^\vee \in \mathbb{R}^3$ is the desired angular velocity. The derivative of the quaternion error is $\dot{\tilde{q}}_0 = -1/2\tilde{\bm{q}}_v^T\tilde{\bm{\omega}}$ and $\dot{\tilde{\bm{q}}}_v = 1/2\bm{Q}\tilde{\bm{\omega}}$, where $\bm{Q} = \tilde{q}_0\bm{E}_3 + [\tilde{\bm{q}}_v]_\times \in \mathbb{R}^{3\times3}$, with $\bm E_3 \in \mathbb{R}^{3 \times 3}$ denoting the identity matrix. The second derivative of $\tilde{\bm{q}}_v$ is $\ddot{\tilde{\bm{q}}}_v = 1/2\dot{\bm{Q}}\tilde{\bm{\omega}} + 1/2\bm{Q}\dot{\tilde{\bm{\omega}}}$, where $\dot{\bm{Q}} = \dot{\tilde{q}}_0\bm{E}_3 + [\dot{\tilde{\bm{q}}}_v]_\times$. Substituting the attitude dynamics \eqref{eq_att_dynamics} into the expression of $\ddot{\tilde{\bm{q}}}_v$ yields
\begin{equation}
\begin{aligned}
\ddot{\tilde{\bm{q}}}_v 
&= \frac{1}{2}\dot{\bm{Q}}\tilde{\bm{\omega}} + \frac{1}{2}\bm{Q}\bm{I}^{-1}(\bm{\tau} - \bm{\omega} \times \bm{I}\bm{\omega}) \\
&\quad - \frac{1}{2}\bm{Q}\dot{\tilde{\bm{R}}}\bm{\omega}_d - \frac{1}{2}\bm{Q}\tilde{\bm{R}}\dot{\bm{\omega}}_d + \frac{1}{2}\bm{Q}\bm{\Delta}_\omega,
\end{aligned}
\end{equation}
which can be rewritten as
\begin{equation}
\label{eq_att_error_dynamics}
\ddot{\tilde{\bm{q}}}_v = \bm{f}(\tilde{\bm{q}}_v, \dot{\tilde{\bm{q}}}_v) + {1}/{2} {\bm Q }\bm I^{-1}\bm \tau +  1/2 {\bm Q }\bm \Delta_\omega,
\end{equation}
where $\bm{f}(\tilde{\bm{q}}_v, \dot{\tilde{\bm{q}}}_v) = 1/2\dot{\bm{Q}}\tilde{\bm{\omega}} - 1/2\bm{Q}\bm{I}^{-1}(\bm{\omega} \times \bm{I}\bm{\omega}) - 1/2\bm{Q}\dot{\tilde{\bm{R}}}\bm{\omega}_d - 1/2\bm{Q}\tilde{\bm{R}}\dot{\bm{\omega}}_d$ denotes the nonlinear term.
The design of the attitude control law consists of two steps.

The first is to design a sliding mode vector to guide the error trajectory toward the preset error trajectory. 
The upper bound function for the attitude performance envelope is denoted by $\bm \rho_{q}(t)\in \mathbb{R}^3$, while the preset error trajectory is denoted by $\bm \beta_q(t) \in \mathbb{R}^3$, which is designed according to Section~\ref{sec_preset_error_trajectory}. Let $\bm \rho_{q,0}\in \mathbb{R}^3$, $\bm \rho_{q,\infty}\in \mathbb{R}^3$, and $l_q\in \mathbb{R}$ be the parameters of $\bm \rho_{q}(t)$. 
Let $\bm c_q \in \mathbb{R}^3$  and $\bm b_q \in \mathbb{R}^3$ be the parameters of $\bm \beta_q(t)$. Define $\bm z_q(t)=\tilde{\bm q}_v(t)- \bm \beta_q(t) \in \mathbb{R}^3$. Then the sliding mode vector is designed as 
\begin{equation}
    \begin{aligned}
    {\bm s}_q =\dot {\bm{z}}_q + \bm{\Lambda}_q {\bm z}_q,
    \end{aligned}
\end{equation}
where $\bm{\Lambda}_q \in \mathbb{R}^{3\times3}$ is a positive diagonal matrix.

The second is to design the tracking control law. \blue{Combining with \eqref{eq_att_error_dynamics},  the time derivative of  $\bm s_q$ is given by}
\begin{equation}
\begin{aligned}
\label{eq_dot_sq}
\dot {\bm s}_q  &=\ddot {\bm{z}}_q + \bm{\Lambda}_q\dot {\bm z}_q\\
 &= \ddot{\tilde{\bm q}}_v - \ddot{\bm \beta}_{q} + \bm{\Lambda}_q\dot {\bm z}_q \\
&= \bm f_q(\tilde{{\bm q}}_v,\dot {\tilde{{\bm q}}}_v) +  \frac{1}{2} {\bm Q }\bm I^{-1}\bm \tau +  \frac{1}{2} {\bm Q }\bm \Delta_\omega -\ddot{\bm \beta}_{q} + \bm{\Lambda}_q\dot {\bm z}_q.
\end{aligned}
\end{equation}
Then, the control torque input  is designed as
\begin{equation}
\label{eq_att_controller}
\bm \tau=2\bm I \bm Q^{-1} (- \bm f_q(\tilde{{\bm q}}_v,\dot {\tilde{{\bm q}}}_v) - \frac{1}{2} {\bm Q }\hat{\bm \Delta}_\omega+ \ddot{\bm \beta}_q-\bm{\Lambda}_q \dot{\bm z}_q - \bm K_q\bm s_q),  
\end{equation}
where $\bm K_q \in \mathbb{R}^{3\times 3}$ is a positive diagonal matrix. \blue{The matrix $\bm{Q}$ is invertible in the control computation. This is because its definition dictates that singularity occurs only when the attitude tracking error reaches 180$^\circ$, which does not arise in practical control scenarios.} The proof that the position tracking error $\tilde{\bm q}_v$ always remains within the performance envelope is given as follows.
\blue{\begin{theorem}
\label{thm_att_error}
Let $\lambda_{\min}(\bm K_q)$ be the minimum eigenvalue of $\bm K_q$, and $[\Lambda_q]_{i,i}$ the $i$-th diagonal element of $\bm \Lambda_q$. Assuming that $\bm \Delta_\omega$ and its derivative $\dot{\bm \Delta}_\omega$ are bounded, for the $i$-th element $\tilde{ q}_{v,i}(t)$ of the tracking error $\tilde{\bm q}_v(t)$, if the parameters are chosen such that ${\delta_{\omega,f}}/([\Lambda_q]_{i,i}\lambda_{\min}(\bm K_q)) < \min\{\rho_{q,\infty,i}, \rho_{q,0,i} - |\tilde{q}_{v,i}(0)| \}$ and $c_{q,i} > {|b_{q,i}|}/(\epsilon_{q,0,i}-{\delta_{\omega,f}}/([\Lambda_q]_{i,i}\lambda_{\min}(\bm K_q)))$, then the attitude tracking error satisfies $ | \tilde{ q}_{v,i}(t) |< {\rho}_{q,i}(t), \ \forall t \ge 0$.
\end{theorem}}

\begin{proof}
\blue{Substituting \eqref{eq_att_controller} into \eqref{eq_dot_sq}, the closed-loop dynamics of the sliding mode vector $\bm s_q$ is given by $\dot{\bm s}_p =-\bm K_q \bm s_q + {1}/{2}\bm Q\tilde{\bm \Delta}_\omega$.
Because the unit quaternion satisfies $\|\tilde q_0 \bm E_3\| = |\tilde q_0| \le 1$ and $\|[\tilde{\bm q}_v]_\times\| = \|\tilde{\bm q}_v\| \le 1$, one has ${1}/{2}\|\bm Q \tilde{\bm \Delta}_\omega\| \le {1}/{2} \|\bm Q\|\cdot\|\tilde{\bm \Delta}_\omega\| 
\le {1}/{2} (|\tilde q_0| + \|[\bm {\tilde q}_v]_{\times}\|)\|\tilde{\bm \Delta}_\omega\| \le \|\tilde{\bm \Delta}_\omega\| \le \delta_{\omega}$, where $\delta_\omega$ is the upper bound of the estimation error $\tilde{\bm \Delta}_\omega$ for all $t\ge 0$. Define a Lyapunov function as $V = 1/2\bm s^T_q \bm s_q$. The time derivative of $V$ is $\dot{V} ={\bm s}_q^T \dot{\bm s}_q$. 
Substituting \eqref{eq_dot_sq} into the expression of $\dot V$ yields 
\begin{equation}
\label{eq_V_dot_analysis_q}
\begin{aligned}
\dot{V} &= \bm s^T_q (-\bm K_q \bm s_q + \frac{1}{2}\bm Q\tilde{\bm \Delta}_{\omega}) \\
&\leq -\lambda_{\min}(\bm K_q) \|\bm s_q\|^2 + \frac{1}{2}\|\bm s_q\| \|\bm Q\| \|\tilde{\bm \Delta}_\omega\| \\
& \leq -\lambda_{\min}(\bm K_q) \|\bm s_q\|^2 + \delta_{\omega} \|\bm s_q\|.\\
\end{aligned}
\end{equation}
According to Theorem~4.18 in \cite{khalil2002nonlinear} and  $\bm s_q(0) = \bm 0$, the sliding mode vector satisfies $|\bm s_q(t)| \le {\delta_\omega}/{\lambda_{\min}(\bm K_q)}, \ \forall t \ge 0$.
Consequently, the $i$-th element $s_{q,i}(t)$ of $\bm s_q(t)$ satisfies $|s_{q,i}(t)| \leq \delta_\omega / \lambda_{\min}(\bm K_q)$. From the definition of $\bm s_q$, we have $s_{q,i} = \dot{z}_{q,i} + [\Lambda_q]_{i,i} z_{q,i}$. According to Lemma~1 in \cite{lemma_z} and the upper bound of $s_{q,i}(t)$, one can conclude that
\begin{equation}
\label{eq_att_deviation_to_preset_trajectory}
|z_{q,i}(t)| \le \frac{\delta_{\omega}}{[\Lambda_q]_{i,i}\lambda_{\min}(\bm K_q)}, \  \forall t\ge 0.
\end{equation}
}

\blue{Further, we prove that $ | \tilde{q}_{v,i}(t) |< {\rho}_{q,i}(t), \ \forall t \ge 0$. Combining \eqref{eq_att_deviation_to_preset_trajectory} with the definition of $\bm z_q$, the attitude tracking error satisfies
\begin{equation}
\begin{aligned}|\tilde q_{v,i}(t)| &=|\tilde{q}_{v,i}(t)+ \beta_{q,i}(t)- \beta_{q,i}(t)|\\
&\le |\tilde {q}_{v,i}(t) - \beta_{q,i}(t)|+| \beta_{q,i}(t)|\\
&\le \frac{\delta_{\omega}}{[\Lambda_q]_{i,i}\lambda_{\min}(\bm K_q)}  + | \beta_{q,i}(t)|,
\end{aligned}
\label{eq_att_bound_tracking_error}
\end{equation}
as $\forall t\ge 0, \ i\in\{1,2,3\}$. According to Section~\ref{sec_att_ESO}, $\delta_\omega$ is defined as a piecewise constant: $\delta_\omega = \delta_{\omega,b}$ for $t < t_{\omega,f}$ and $\delta_\omega = \delta_{\omega,f}$ for $t \ge t_{\omega,f}$. Here, $\delta_{\omega,b}$ and $\delta_{\omega,f}$ represent the upper bounds of the estimation error before and after attitude ESO convergence, respectively. Therefore, $\delta_{\omega,b}$ is much greater than $\delta_{\omega,f}$. The control parameters are set such that ${\delta_{\omega,f}}/([\Lambda_q]_{i,i}\lambda_{\min}(\bm K_q)) < \min\{\rho_{q,\infty,i}, \rho_{q,0,i} - |\tilde{q}_{v,i}(0)| \}$, and the design parameter of $\beta_{q,i}$ is chosen to satisfy $c_{q,i} > {|b_{q,i}|}/(\rho_{q,0,i} - |\tilde{q}_{v,i}(0)|-{\delta_{\omega,f}}/([\Lambda_q]_{i,i}\lambda_{\min}(\bm K_q)))$. Then, according to Section~\ref{sec_preset_error_trajectory}, we obtain $|\beta_{q,i}(t)| < \rho_{q,i}(t) -{\delta_{\omega,f}}/([\Lambda_q]_{i,i}\lambda_{\min}(\bm K_q))$. For the phase $t \ge t_{\omega,f}$, applying the resulting inequality of $|\beta_{q,i}(t)|$ to \eqref{eq_att_bound_tracking_error} yields $|\tilde q_{v,i}(t)| < {\delta_{\omega,f}}/({[\Lambda_q]_{i,i}\lambda_{\min}(\bm K_q)}) + \rho_{q,i}(t) -{\delta_{\omega,f}}/({[\Lambda_q]_{i,i}\lambda_{\min}(\bm K_q)}) = \rho_{q,i}(t)$. This means the tracking error remains within the performance envelope for the phase $t \geq t_{\omega,f}$. For the phase $t < t_{\omega,f}$, this inclusion can be guaranteed through the design of the attitude performance envelope parameter. Specifically, we set $\rho_{q,0,i} > ( |\tilde{q}_{v,i}(0)| + {\delta_{\omega,b}}/({[\Lambda_q]_{i,i}\lambda_{\min}(\bm K_q)}) - \rho_{q,\infty,i} ) \exp(l_q t_{\omega,f}) + \rho_{q,\infty,i}$. Combining the specified set of parameters $\rho_{q,0,i}$ and the definition of $\rho_{q,i}(t)$, we obtain $ \rho_{q,i}(t_{\omega,f}) > |\tilde{q}_{v,i}(0)| + {\delta_{\omega,b}}/({[\Lambda_q]_{i,i}\lambda_{\min}(\bm K_q)})$. From the monotonic decreasing property of $|\beta_{q,i}(t)|$ and $\rho_{q,i}(t)$, one can conclude that $|\beta_{q,i}(t)| \le |\beta_{q,i}(0)|= |\tilde{q}_{v,i}(0)|$ and $\rho_{q,i}(t) > \rho_{q,i}(t_{\omega,f})$, when $t< t_{\omega,f}.$ Then, we have $ |\beta_{q,i}(t)| \le |\tilde{q}_{v,i}(0)|< \rho_{q,i}(t) -{\delta_{\omega,b}}/({[\Lambda_q]_{i,i}\lambda_{\min}(\bm K_q)})$.
Substituting this result into \eqref{eq_att_bound_tracking_error} yields $|\tilde q_{v,i}(t)|< \rho_{q,i}(t)$, when $t < t_{\omega,f}$. Combining the proof results from both phases, it can be concluded that the tracking error remains within the attitude performance envelope.}
\end{proof}

\blue{Theorem~\ref{thm_att_error} assumes that $\bm {\Delta}_\omega$ and $\dot{\bm {\Delta}}_\omega$ are bounded. This assumption is reasonable because $\bm {\Delta}_\omega$ and $\dot{\bm {\Delta}}_\omega$ depend on the motion of the quadcopter base and the robotic arm, which are physically constrained.} The analysis of the influence of parameters on the deviation between the attitude tracking error and the preset error trajectory is given below. The upper bound of the deviation is given in \eqref{eq_att_deviation_to_preset_trajectory}. It shows that the deviation is inversely proportional to $[\Lambda_q]_{i,i}$ and $\lambda_{\min}(\bm K_q)$, and directly proportional to $\delta_{\omega,f}$. \blue{Consequently, to reduce this deviation, one can increase the control gains $\bm{\Lambda}_q$ and $\bm{K}_q$, and reduce the parameter $\bm \varepsilon_\omega$ to lower the upper bound $\delta_{\omega,f}$ of the attitude ESO estimation error.}

\begin{figure*}[t]
	\centering
	\includegraphics[width=1\textwidth]{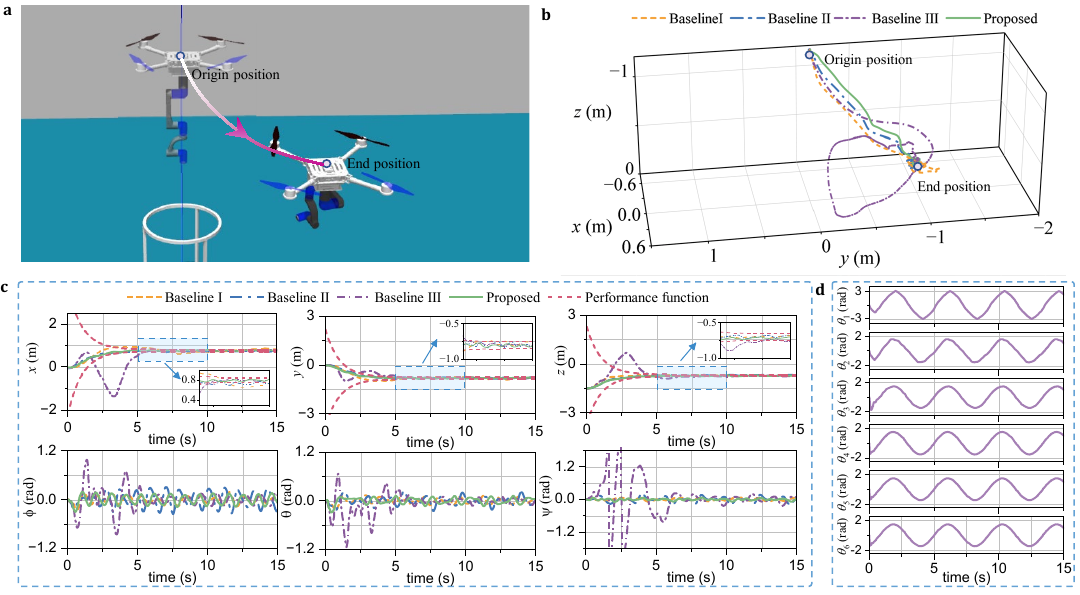}
	\caption{Results of the setpoint tracking experiment. (a) Visual description of the experiment. (b) The trajectories of the quadcopter base. (c) Position and attitude tracking results of the quadcopter base. (d) Tracking results of robotic arm joints.}
	\label{fig_setpoint_tracking}
\end{figure*}

\begin{figure*}[t]
	\centering
	\includegraphics[width=1\textwidth]{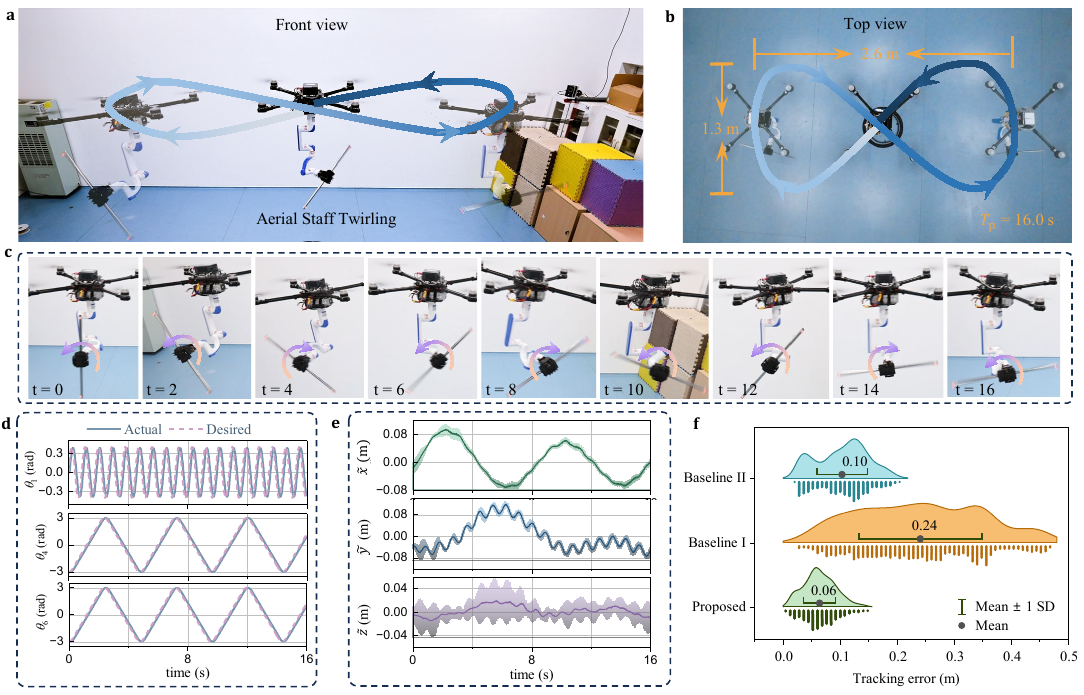}
	\caption{Results of the aerial staff twirling experiment. \blue{(a) Visual description of the experiment. The reference trajectory follows a figure-eight curve defined by $x= 0.65\sin (4\pi t/T_p)$ and $y = 1.3\sin(2\pi t/T_p)$, with the trajectory period set to $T_p=16$~s. (b) Top view of the experiment.  (c) Snapshots of the staff twirling. (d) Tracking results of the 1st, 4th, and 6th joints of the robotic arm. (e) Position tracking errors of the quadcopter base. The error plots depict the mean (solid line) and standard deviation (shaded region) calculated from ten repeated trials. (f) Raincloud plot of position tracking error distributions for the quadcopter base under three control methods.}}
	\label{fig_aerial_staff_twirling}
\end{figure*}

\begin{figure*}[t]
	\centering
	\includegraphics[width=1\textwidth]{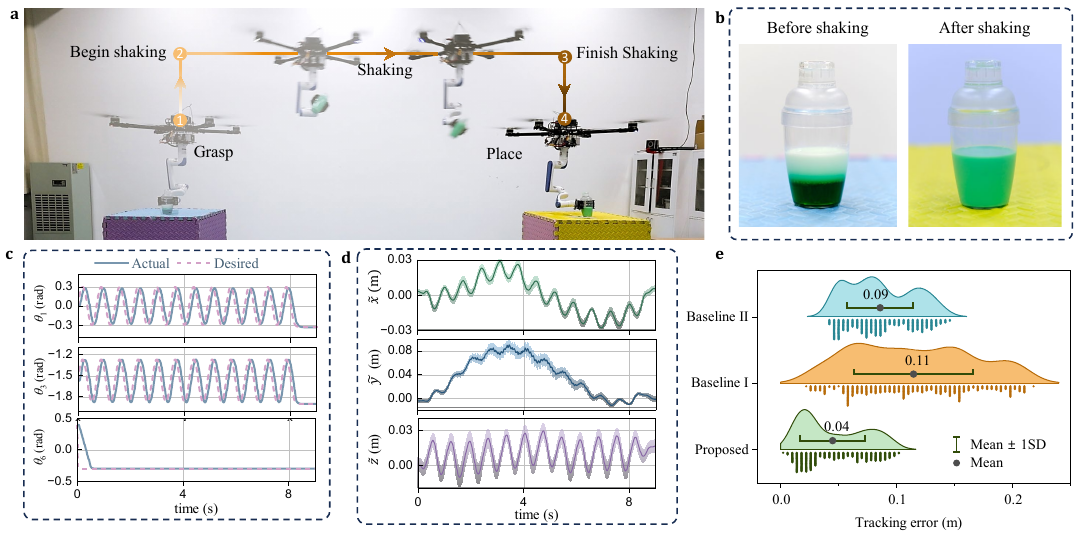}
	\caption{Results of the aerial mixology experiment. (a) Visual description of the experiment. (b) Pre-shake layered materials and post-shake cocktail. (c) Tracking results of the 1st, 2nd, and 6th robotic arm joints. (d) Quadcopter position tracking errors during shaking. The error plots depict the mean (solid line) and standard deviation (shaded region) calculated from ten repeated trials. \blue{(e) Raincloud plot of position tracking error distributions for the quadcopter base under three control methods.}}
	\label{fig_aerial_mixology}
\end{figure*}

\begin{figure*}[t]
	\centering
	\includegraphics[width=1\textwidth]{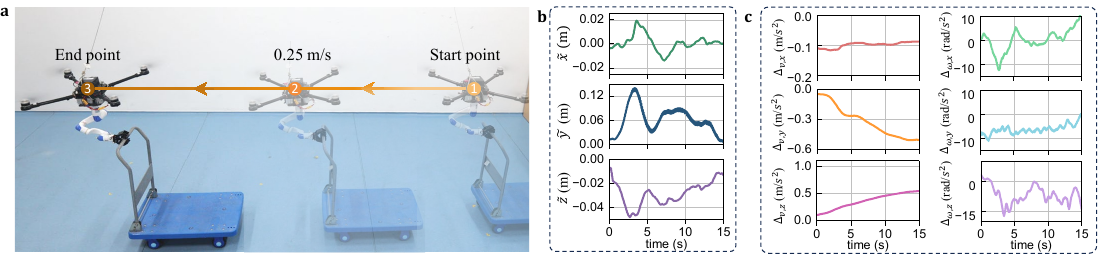}
\caption{Results of the aerial cart-pulling experiment. (a) Visual description of the experiment. (b) Position tracking errors of the quadcopter base throughout the entire pulling process.  (c) Estimation results of the position ESO (left) and attitude ESO (right). }
	\label{fig_aerial_cart_pulling}
\end{figure*}

\section{\blue{Numerical Validation}}

The aerial manipulator platform used in the simulation consists of a quadcopter and a six-DOF robotic arm. The quadcopter has a mass of 5.40~kg and a wheelbase of 0.90~m. The robotic arm has a mass of 2.32~kg. Its workspace resembles a hemisphere with a radius of 0.38~m centered at the robotic arm base.

The same set of control parameters is used in both the simulation and the experiments on the real platform. For position control, the parameters of $\bm \rho_p$ are set to $\bm \rho_{p,0} = [3,3,3]^{T}$,  $\bm \rho_{p,\infty} = [0.05,0.05,0.05]^{T}$, and $l_p = 1$. 
The parameters of $\bm \beta_p$ are selected as $\bm c_{p} = [5,5,5]^T$. The position ESO parameters are set as \blue{$\bm \alpha_{v} =[0.1,0.1,0.1]^T$~s$^{-1}$ and $\bm \varepsilon_v=[0.5,0.5,0.5]^T$. The control gains are $\bm{K}_{p}=\text{diag}([3,3,3])$~s$^{-1}$ and $\bm{\Lambda}_{p}=\text{diag}([3,3,3])$~s$^{-1}$. }For attitude control, the parameters of $\bm \rho_q$ are set to  $\bm \rho_{q,0} = [0.4,0.4,0.4]^{T}$, $\bm \rho_{q,\infty} = [0.06,0.06,0.06]^{T}$, and $l_q = 1$. The parameters of $\bm \beta_q$ are selected as $\bm c_{q} = [5,5,5]^T$. \blue{The attitude ESO parameters are set as $\bm \alpha_{\omega} =[1,1,1]^T$~s$^{-1}$ and $\bm \varepsilon_q=[0.25,0.25,0.25]^T$. The control gains are $\bm{K}_{q}=\text{diag}([12,12,10])$~s$^{-2}$ and $\bm{\Lambda}_{q}=\text{diag}([14,14,10])$~s$^{-2}$. }For both position and attitude control, the parameters of the nonlinear gain function $g(e)$ are chosen as $w=0.5$ and $d=5$.

To simulate strong dynamic coupling, the robotic arm performs fast swinging motions, as shown in Fig.~\ref{fig_setpoint_tracking}d. In this experiment, the end-effector can reach a translational speed of 0.76~m/s and an angular velocity of 9.28~rad/s relative to the quadcopter base. To demonstrate the effectiveness of the proposed algorithm, three baseline methods are compared. Baseline I is \blue{the PX4 1.13.0 flight controller, selected due to its widespread adoption in aerial manipulators \cite{gomez2018methods_px4}.} Baseline II shares the same architecture as the proposed method, but does not include compensation for dynamic coupling via the variable-gain ESOs. Baseline III also utilizes the same architecture as the proposed method, but excludes the preset error trajectory. \blue{We selected these three baseline methods because they represent the prevailing control strategies for aerial manipulators: standard PID control, intrinsic robust nonlinear control, and observer-based nonlinear control.}

\begin{table}[htbp]
\centering
\caption{  Quantitative Comparison of Position Tracking Errors (Unit: cm)}
\label{tab_comparison}
\renewcommand{\arraystretch}{1.1} 
\begin{tabular}{llcc}
\toprule
\textbf{Scenarios} & \textbf{Method} & \textbf{Mean $\pm$ SD} & \textbf{Maximum} \\
\midrule
\multirow{4}{*}{Setpoint} 
 & Baseline I   & $5.54 \pm 2.94$ & 17.13 \\
 & Baseline II  & $5.29 \pm 2.18$ & 12.76 \\
 & Baseline III & $2.63 \pm 0.96$ & 5.48  \\
 & \textbf{Proposed} & $\mathbf{2.35 \pm 1.03}$ & $\mathbf{4.57}$ \\
\midrule
\multirow{4}{*}{Circle} 
 & Baseline I   & $6.13 \pm 2.84$ & 13.88 \\
 & Baseline II  & $3.57 \pm 1.80$ & 9.26  \\
 & Baseline III & $2.82 \pm 1.08$ & 6.59  \\
 & \textbf{Proposed} & $\mathbf{2.78 \pm 0.96}$ & $\mathbf{6.14}$ \\
\midrule
\multirow{4}{*}{Figure-eight} 
 & Baseline I   & $8.86 \pm 4.51$ & 19.42 \\
 & Baseline II  & $6.53 \pm 2.96$ & 14.96 \\
 & Baseline III & $4.12 \pm 1.32$ & 8.02  \\
 & \textbf{Proposed} & $\mathbf{4.04 \pm 1.32}$ & $\mathbf{7.87}$ \\
\bottomrule
\end{tabular}
\end{table}

Figs.~\ref{fig_setpoint_tracking}b-c show the tracking performance of the four methods in the setpoint tracking task. \blue{In this task, the quadcopter base flies from the initial position $(0, 0, -1.5)$~m to the target position $(0.8, -0.8, -0.7)$~m. } The steady-state mean tracking error is defined as the mean position error of the quadcopter base after the system first enters the steady-state boundary $\bm \rho_{p,\infty}$. The steady-state mean tracking position errors for the three baseline methods are 5.54~cm, 5.29~cm, and 2.63~cm, respectively, while that of the proposed method is 2.35~cm. The proposed method reduces the tracking errors by 57.6\%, 52.7\%, and 10.7\%, respectively. This result demonstrates that the proposed method achieves higher accuracy after entering the steady-state boundary. This is because the ESO can compensate for the coupling disturbances and torques induced by the fast swinging of the robotic arm. The tracking error of Baseline III is significantly larger before reaching $\bm \rho_{p,\infty}$ (see Fig.~\ref{fig_setpoint_tracking}c). In contrast, the proposed method strictly confines the error within the prescribed performance boundary. This demonstrates that the introduction of the preset error trajectory can indeed constrain the actual tracking error within the performance envelope.

\blue{Furthermore, we present statistical results for three command tracking tasks: setpoint tracking, circular trajectory tracking (radius: $1.5 \text{ m}$, period: $16 \text{ s}$), and figure-eight trajectory tracking (major axis: $3.0 \text{ m}$, period: $16 \text{ s}$). In all tasks, the robotic arm follows the same fast-swinging motion as shown in Figs.~\ref{fig_setpoint_tracking}d. A detailed quantitative comparison of the three scenarios is summarized in Table~\ref{tab_comparison}. The proposed method achieves the highest precision across all tasks. Specifically, it yields the lowest overall mean tracking error in all three dimensions for setpoint, circular, and figure-eight tracking, with values of 2.35~cm, 2.78~cm, and 4.04~cm, respectively. These results demonstrate that the proposed method effectively suppresses tracking deviations caused by strong dynamic coupling and exhibits significantly enhanced robustness compared to the baseline methods.}

\section{\blue{Experimental Verification}}

\blue{This section presents the experimental results validating the proposed method. }

\subsection{\blue{Experimental Setup}}

The aerial manipulator platform used in the experiments consists of a quadcopter, a Ti5 Robot Eblm-1 robotic arm, and a rigid two-finger gripper (see Fig.~\ref{fig_main_mixology}). \blue{The quadcopter has a mass of 5.40~kg and a wheelbase of 0.90~m. The robotic arm has a mass of 2.32~kg. The gripper has a mass of 206~g. The proposed flight control runs on a Pixhawk 4 autopilot. The trajectory for each task is computed within the ROS environment on an onboard Intel NUC i7 computer.} \blue{The control commands for the robotic arm and the gripper are executed via two independent paths. For the robotic arm, commands are first sent from the onboard computer to its embedded controller via serial communication (UART protocol), and then forwarded to its six motors through a CAN bus. For the gripper, commands are transmitted wirelessly from the onboard computer to the gripper’s embedded controller using the TCP protocol, and the controller subsequently drives the DS3218 servo with the generated PWM signals.}

The experiments are conducted in a Vicon system, which provides accurate position measurements of the quadcopter base and the end-effector of the aerial manipulator. \blue{The measurement data of the Vicon system is sent to a ground control station through an Ethernet switch. A 5 GHz wireless router is used to connect the ground control station and the aerial manipulator. The ground control station sends the measurement data and the control command to the aerial manipulator at a frequency of 100 Hz. The aerial manipulator sends the state data to the ground control station with a frequency of 50 Hz.}

\subsection{Example 1: Aerial Staff Twirling}
The objective of this experiment is to validate that the proposed method achieves high-precision trajectory tracking under strong dynamic coupling. Compared to the previous example, this experiment is more challenging. The quadcopter base tracks a figure-eight trajectory (Fig.~\ref{fig_aerial_staff_twirling}a,b) while the robotic arm joints swing rapidly as shown in Fig.~\ref{fig_aerial_staff_twirling}d. Additionally, a staff mounted on the robotic arm’s end-effector is continuously rotated at a speed of 5.6~rad/s during flight (see Fig.~\ref{fig_aerial_staff_twirling}c,d).
The end effector can reach a speed of 0.51~m/s relative to the quadcopter base in the experiment. 

Fig.~\ref{fig_aerial_staff_twirling}e shows position tracking errors of the quadcopter base across ten repeated trials. The mean position errors in the $x$-, $y$-, and $z$- directions are 4.24~cm, 3.97~cm, and 0.79~cm, with standard deviations of 4.86~cm, 4.86~cm and 0.95~cm, respectively. \blue{ To further highlight the advantage of the proposed algorithm, we employed a comparative raincloud plot to visually contrast the performance of the proposed method with that of Baseline~I and Baseline~II (see Fig.~\ref{fig_aerial_staff_twirling}f). The results show that the error distribution of the proposed method is highly concentrated. Moreover, it exhibits the smallest overall mean tracking error in all three dimensions, which is 0.06~m. In contrast, Baseline I and Baseline II exhibit significantly more dispersed error distributions, with larger mean errors of 0.24~m and 0.10~m, respectively. These findings clearly demonstrate the proposed method performs better in the aerial staff twirling experiment than the two baseline methods. }

\subsection{Example 2: Aerial Mixology}

This experiment aims to verify that the proposed method enables the aerial manipulator to perform an aerial mixology task. In this experiment, we prepare a cocktail using an aerial manipulator (see Fig.~\ref{fig_aerial_mixology}a). The aerial mixology experiment consists of three steps. First, the aerial manipulator flies to a position and grasps the shaker that contains the milk and Crème de Menthe. Second, the aerial manipulator flies towards a target location while the robotic arm executes rapid shaking motions to thoroughly mix the materials (see Fig.~\ref{fig_aerial_mixology}b). Third, the prepared cocktail is placed at the designated position. The combined mass of the shaker and the cocktail is 220~g, and the mass of the gripper is 206~g. In the mixology process, the end-effector can achieve a velocity of 1.02~m/s and an acceleration of~5.10 m/s$^2$ relative to the quadcopter base.

Fig.~\ref{fig_aerial_mixology}d shows the quadcopter base's tracking errors over ten repeated trials. During the shaking phase, the mean position errors in the $x$-, $y$-, and $z$- directions are 1.25~cm, 4.04~cm, and 0.96~cm, with standard deviations of 1.49~cm, 3.25~cm, and 1.04~cm, respectively. 
\blue{Moreover, the raincloud plot in Fig.~\ref{fig_aerial_mixology}e demonstrates that the proposed method significantly outperforms Baseline~I and Baseline~II. The proposed method exhibits a highly concentrated error distribution, with an overall mean tracking error in three directions of 0.04 m. In contrast, the two baseline methods show more dispersed error distributions, with overall mean tracking errors in three directions of 0.11 m and 0.09 m, respectively. These findings confirm that the proposed method achieves higher tracking accuracy in the presence of unknown payloads and dynamic coupling.}

\subsection{Example 3: Aerial Cart-Pulling}

\blue{This experiment is designed to validate the disturbance rejection performance of the proposed method. During the experiment, the aerial manipulator first flies to a position directly above the handle of the cart and uses a gripper to grasp it. The quadcopter base then moves leftward with a velocity of 0.25~m/s, pulling the cart over a distance of 3.2~m (see Fig.~\ref{fig_aerial_cart_pulling}a). The cart weighs 14.6~kg. Throughout the pulling process, ground friction is generated, which introduces disturbances to the motion of the aerial manipulator.}

\blue{ Fig.~\ref{fig_aerial_cart_pulling}b shows the position tracking errors of the quadcopter base during the pulling process. When pulling begins, the error along the y-axis (pulling direction) rises rapidly. This occurs because the cart exerts a significant disturbance force and torque on the quadcopter base. The position and attitude ESOs estimate the disturbances, and feed the estimation results back to the flight controller for disturbance rejection. Fig.~\ref{fig_aerial_cart_pulling}c presents the estimation results of these two ESOs. The mean tracking errors in the three directions are 0.47~cm, 6.21~cm, and 2.88~cm, respectively. This experiment demonstrates that despite the aerial manipulator’s end-effector being subjected to significant external disturbances, the proposed method is still able to accomplish the task successfully.}

\section{Conclusion}

This paper presents a prescribed performance motion control framework for aerial manipulators. The approach uses variable-gain ESOs to estimate dynamic coupling. In addition, a preset error trajectory is introduced to constrain the actual tracking error within the performance envelope. Four experiments were conducted to validate the proposed method. First, an ablation study was conducted to validate the improvement in tracking accuracy contributed by the variable-gain ESOs and the preset error trajectory. Second, in the aerial staff twirling experiment, the aerial manipulator rapidly rotated a staff while the quadcopter base accurately tracked a complex lemniscate (figure-eight) trajectory. The high tracking precision further demonstrates the effectiveness of the proposed algorithm in scenarios involving aggressive robotic arm motion. Third, an aerial mixology experiment was performed, in which the aerial manipulator successfully mixed a cocktail and transported the shaker from one location to a designated position. \blue{Finally, in the aerial cart-pulling experiment, the aerial manipulator successfully pulled a cart leftward, verifying the proposed method's capability to reject external disturbances exerted on the end-effector.}
These results further verify the robustness of the proposed algorithm. Although the proposed algorithm has achieved high performance, its performance may decline when the end-effector interacts with the environment due to unmodeled contact-induced disturbances. Therefore, future work will focus on enhancing adaptability through reinforcement learning \cite{zhao2025RLBook} for contact-rich scenarios.

\bibliography{REF} 
\bibliographystyle{ieeetr}

\begin{IEEEbiography}[{\includegraphics[width=1in,height=1.25in,clip,keepaspectratio]{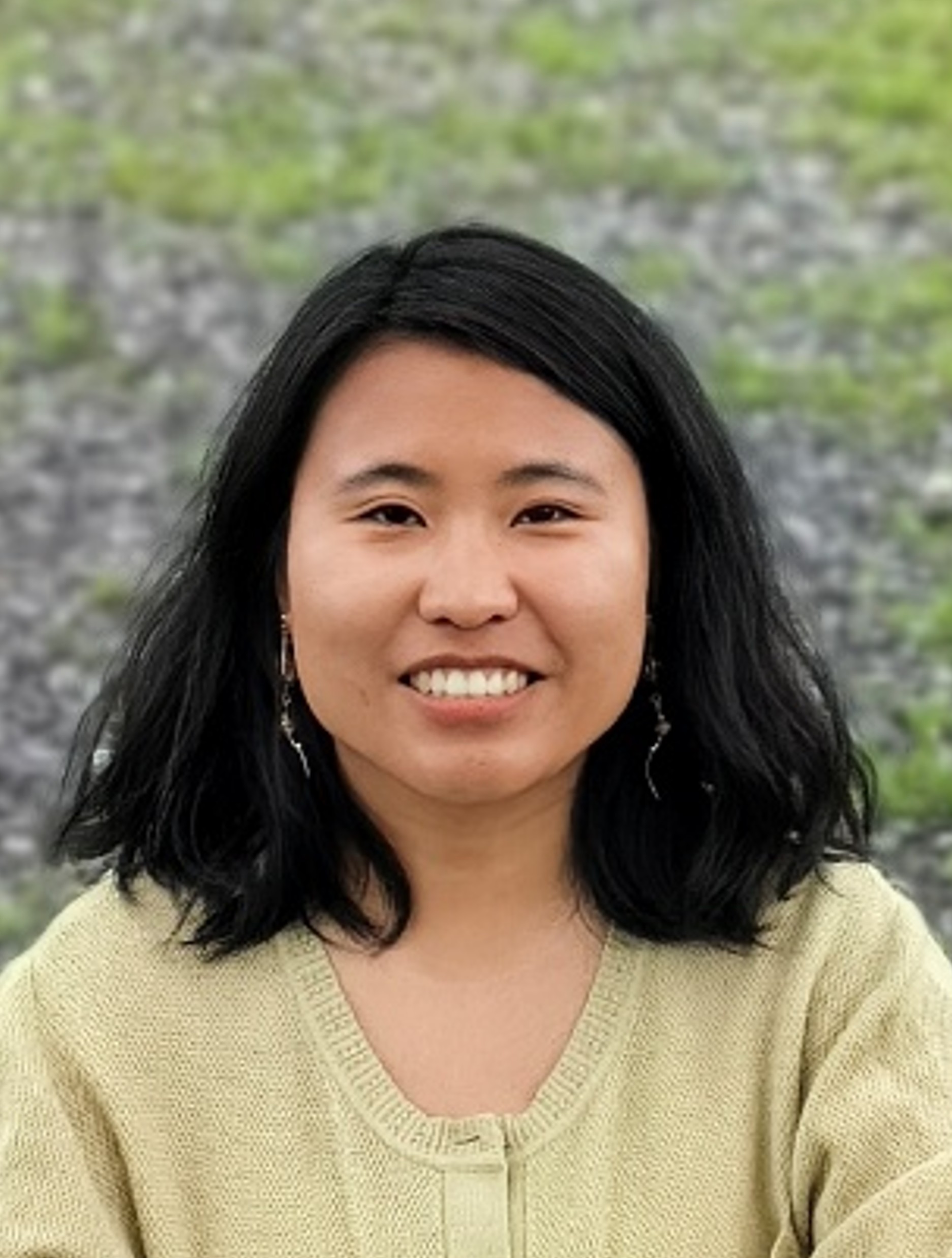}}]
	{Mengyu~Ji} received the B.Eng. degree in Aerospace Engineering from Harbin Engineering University, Harbin, China, in 2017, and the M.Eng. degree from the University of Chinese Academy of Sciences, Beijing, China, in 2020.

    She is currently pursuing the Ph.D. degree with the WINDY Lab, Westlake University, Hangzhou, China. Her research interests include motion control of aerial robotic systems.
\end{IEEEbiography}

\begin{IEEEbiography}[{\includegraphics[width=1in,height=1.25in,clip,keepaspectratio]{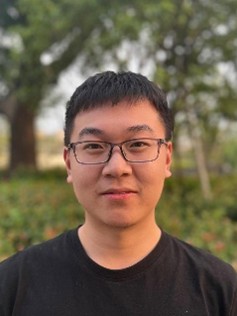}}]
	{Shiliang~Guo}  received the B.Eng. degree in electrical engineering from Tianjin University of Technology and Education, Tianjin, China. 
    
     He was a corecipient of the National Gold Prize in the China University Engineering Practice and Innovation Competition, China, in 2021. He is currently an Embedded Systems Engineer with the WINDY Lab, Westlake University, Hangzhou, China. His research focuses on hardware, software, and controller design of aerial vehicle systems.
\end{IEEEbiography}

\begin{IEEEbiography}[{\includegraphics[width=1in,height=1.25in,clip,keepaspectratio]{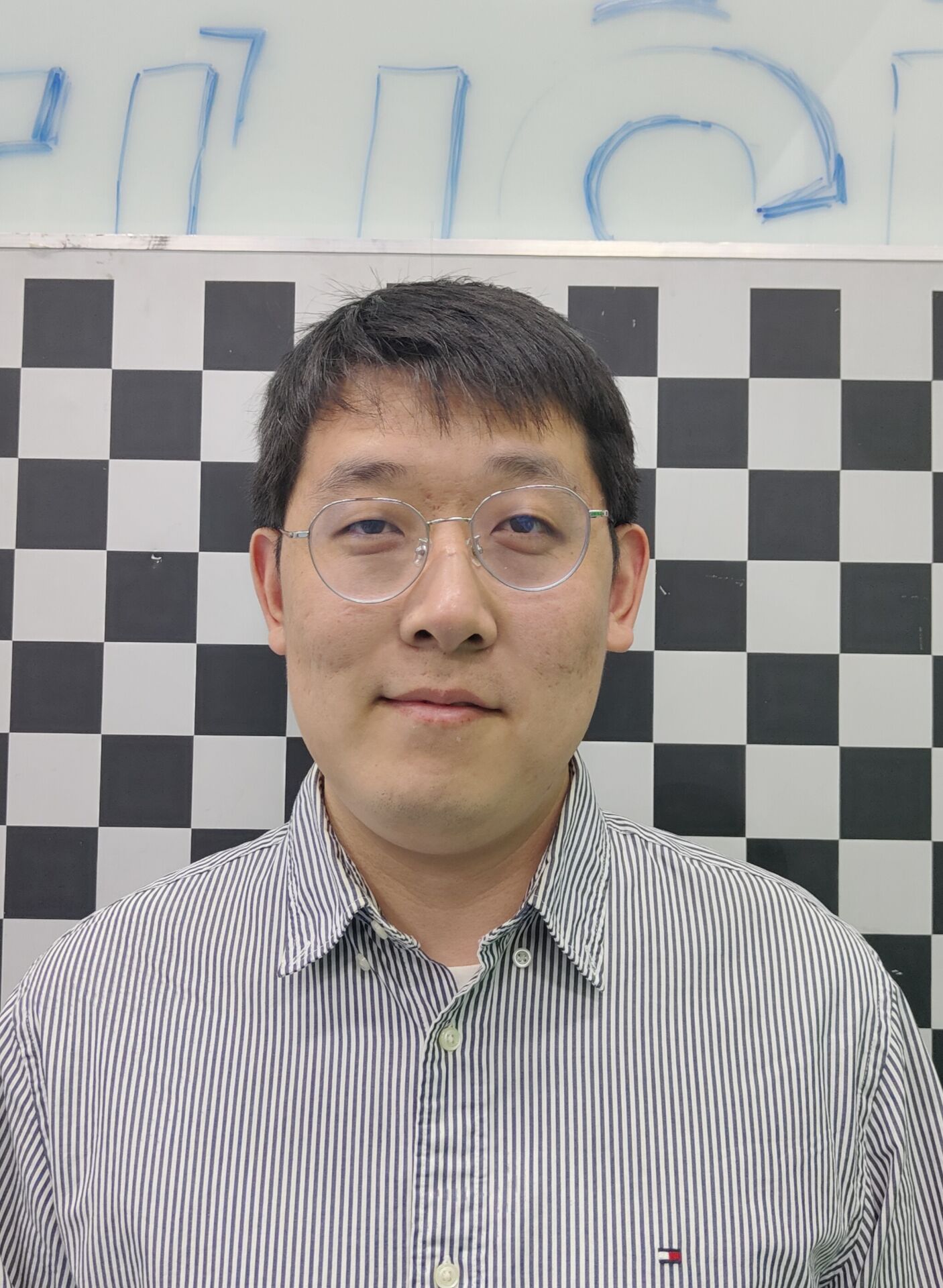}}]
	{Zhengzhen~Li} received the B.Eng. and M.Eng. degrees in biomedical engineering from Shandong University, Jinan, China, in 2019 and 2022, respectively.
    
    He is currently pursuing the Ph.D. degree with the WINDY Lab, Westlake University, Hangzhou, China. His research interests include motion control and reinforcement learning control of aerial robotic systems.
\end{IEEEbiography}

\begin{IEEEbiography}[{\includegraphics[width=1in,height=1.25in,clip,keepaspectratio]{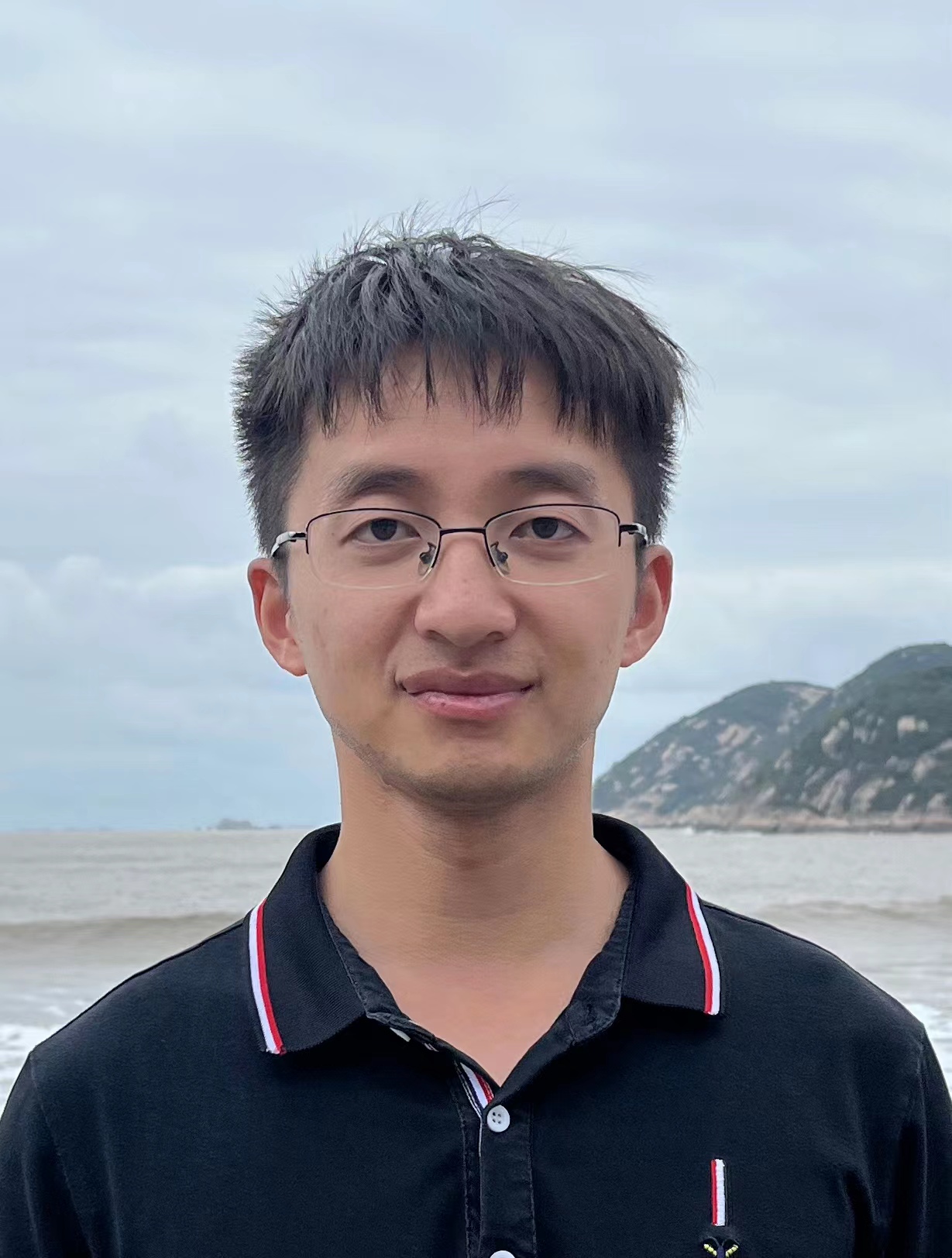}}]
	{Jiahao~Shen}  received the B.Eng. degree from China Jiliang University, Hangzhou, China, in 2019, and the M.Eng. degree from the
	Nanjing University of Aeronautics and Astronautics, Nanjing, China, in 2022, all in electrical engineering.
    
    He is currently a Flight Control Engineer with the WINDY Lab, Westlake University, Hangzhou, China. His research focuses on hardware, software, and controller design of unmanned aerial vehicle systems.
\end{IEEEbiography}

\begin{IEEEbiography}[{\includegraphics[width=1in,height=1.25in,clip,keepaspectratio]{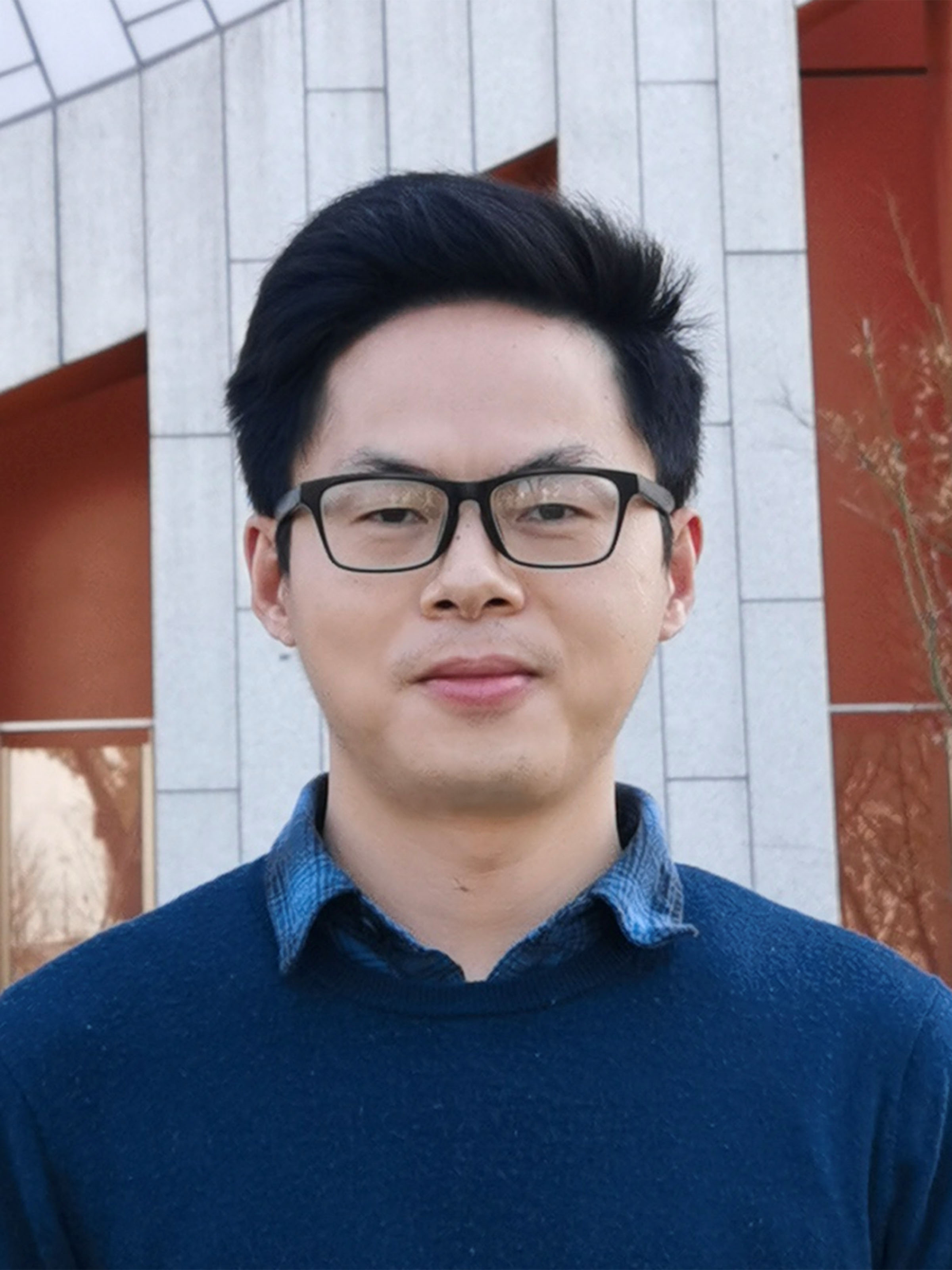}}]
	{Huazi~Cao} received the B.Eng. degree in aerospace engineering from Harbin Engineering University, Harbin, China, in 2013, and the Ph.D. degree in aerospace engineering from the Beijing University of Aeronautics and Astronautics, Beijing, China, in 2018.
    
	He is currently an Associate Research Professor at the Westlake Institute for Optoelectronics, Westlake University, Hangzhou, China. His research interests include control and planning for aerial robots.
\end{IEEEbiography}

\begin{IEEEbiography}[{\includegraphics[width=1in,height=1.25in,clip,keepaspectratio]{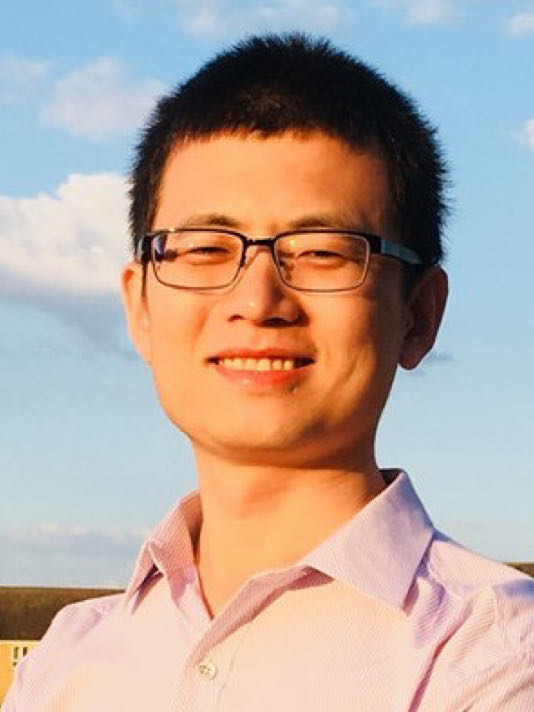}}]
	{Shiyu~Zhao} (Senior Member, IEEE) received the B.Eng. and M.Eng. degrees from the Beijing University of Aeronautics and Astronautics, Beijing, China, in 2006 and	2009, respectively, and the Ph.D. degree from the National University of Singapore, Singapore, in 2014, all in electrical engineering.
	
	From 2014 to 2016, he was a Postdoctoral Researcher with the Technion-Israel Institute of Technology, Haifa, Israel, and the University of California at Riverside, Riverside, CA, USA. From 2016 to 2018, he was a Lecturer with the Department of Automatic Control and Systems Engineering, University of Sheffield, Sheffield, U.K. He is currently an Associate Professor with the School of Engineering, Westlake University, Hangzhou, China. His research focuses on theories and applications of robotic systems.
\end{IEEEbiography}

\end{document}